\documentclass{article}
\pdfoutput=1

\usepackage{graphicx} 
\usepackage{subfigure}

\usepackage[numbers]{natbib}

\usepackage{amsmath}
\usepackage{amssymb}
\usepackage{wrapfig}

\usepackage{algorithm}
\usepackage{algorithmic}

\usepackage{hyperref}

 \usepackage[accepted]{icml2011}

\usepackage{url}

\newtheorem{lemma}{Lemma}
\newtheorem{corollary}{Corollary}
\newtheorem{theorem}{Theorem}

\def\be {\begin{equation}}
\def\ee {\end{equation}}
\def\bea {\begin{eqnarray*}}
\def\eea {\end{eqnarray*}}

\newcommand{\x}{{\mathbf x}}

\newcommand{\w}{{\mathbf w}}

\newcommand{\bF}{\bar{F}}
\renewcommand{\v}{{\mathbf v}}
\renewcommand{\u}{{\mathbf u}}
\newcommand{\e}{{\mathbf e}}

\newcommand{\sgn}{{\mathrm{sgn}}}

\newcommand{\trans}{\dagger}
\newcommand{\supp}{{\mathrm{supp}}}

\newcommand{\BlackBox}{\rule{1.5ex}{1.5ex}}
\newenvironment{proof}{\par\noindent{\bf Proof\ }}{\hfill\BlackBox\\[2mm]}

\newcommand{\reals}{\mathbb{R}}
\newcommand{\Y}{\mathcal{Y}}
\newcommand{\X}{\mathcal{X}}

\newcommand{\E}{\mathbb{E}}
\newcommand{\D}{\mathcal{D}}
\newcommand{\V}{\mathcal{V}}
\newcommand{\inner}[1]{\langle #1 \rangle}

\newcommand{\eqdef}{\stackrel{\mathrm{def}}{=}}

\newcommand{\poly}{\mathrm{poly}}

\DeclareMathOperator*{\argmin}{argmin} 
\DeclareMathOperator*{\argmax}{argmax} 
\DeclareMathOperator*{\prob}{\mathbb{P}}

\renewcommand{\eqref}[1]{eqn.~(\ref{#1})}
\newcommand{\figref}[1]{Fig.~\ref{#1}}
\newcommand{\secref}[1]{Section~\ref{#1}}
\newcommand{\thmref}[1]{Theorem~\ref{#1}}
\newcommand{\lemref}[1]{Lemma~\ref{#1}}

\newcommand{\corref}[1]{Corollary~\ref{#1}}

\newcommand{\indct}[1]{\boldsymbol{1}\!\left[ #1 \right]}

\icmltitlerunning{ShareBoost}

\begin{document}

\twocolumn[
\icmltitle{ShareBoost: Efficient Multiclass Learning with Feature Sharing}

\icmlauthor{Shai Shalev-Shwartz}{shais@cs.huji.ac.il}
\icmladdress{School of Computer Science and Engineering, the Hebrew University of Jerusalem, Israel}
\icmlauthor{Yonatan Wexler}{yonatan.wexler@orcam.com}
\icmladdress{OrCam Ltd., Jerusalem, Israel}
\icmlauthor{Amnon Shashua}{shashua@cs.huji.ac.il}
\icmladdress{OrCam Ltd., Jerusalem, Israel}

\icmlkeywords{Multi-class learning, Mixed-norms, Boosting}

\vskip 0.3in
]

\begin{abstract}
 Multiclass prediction is the problem of classifying an object into a
  relevant target class.  We consider the problem of learning a
  multiclass predictor that uses only few features, and in particular,
  the number of used features should increase sub-linearly with the
  number of possible classes. This implies that features should be
  shared by several classes. We describe and analyze the ShareBoost
  algorithm for learning a multiclass predictor that uses few shared
  features. We prove that ShareBoost efficiently finds a predictor
  that uses few shared features (if such a predictor exists) and that
  it has a small generalization error. We also describe how to use
  ShareBoost for learning a non-linear predictor that has a fast
  evaluation time. In a series of experiments with natural data sets
  we demonstrate the benefits of ShareBoost and evaluate its success
  relatively to other state-of-the-art approaches.
\end{abstract}

\section{Introduction}

Learning to classify an object into a relevant target class surfaces
in many domains such as document categorization, object recognition in
computer vision, and web advertisement. In multiclass learning
problems we use training examples to learn a classifier which will
later be used for accurately classifying new objects. Typically, the
classifier first calculates several features from the input object and
then classifies the object based on those features.  In many cases, it
is important that the runtime of the learned classifier will be
small. In particular, this requires that the learned classifier will
only rely on the value of few features.

We start with predictors that are based on \emph{linear} combinations
of features. Later, in \secref{sec:nonLinear}, we show how our
framework enables learning highly non-linear predictors by embedding
non-linearity in the construction of the features.  Requiring the
classifier to depend on few features is therefore equivalent to
sparseness of the linear weights of features. In recent years, the
problem of learning sparse vectors for linear classification or
regression has been given significant attention. While, in general,
finding the most accurate sparse predictor is known to be NP hard
\citep{Natarajan95,DavisMaAv97}, two main approaches have been
proposed for overcoming the hardness result. The first approach uses
$\ell_1$ norm as a surrogate for sparsity (e.g.  the Lasso algorithm
\cite{Tibshirani96b} and the compressed sensing literature
\cite{CandesTao05,Donoho06a}). The second approach relies on forward
greedy selection of features (e.g. Boosting \cite{FreundSc99b} in the
machine learning literature and orthogonal matching pursuit in the
signal processing community \citep{TroppGi07}).

A popular model for multiclass predictors maintains a weight vector for each one of the classes. In such case, even if the weight vector associated with each class is sparse, the overall number of used features might grow with the number of classes. Since the number of classes can be rather large, and our goal is to learn a model with an overall small number of features, we would like that the weight vectors will share the features with non-zero weights as much as possible.  Organizing the weight vectors of all classes as rows of a single matrix, this is equivalent to requiring sparsity of the \emph{columns} of the matrix.

In this paper we describe and analyze an efficient algorithm for learning a multiclass predictor whose corresponding matrix of weights has a small number of non-zero columns. We formally prove that if there exists an accurate matrix with a number of non-zero columns that grows sub-linearly with the number of classes, then our algorithm will also learn such a matrix. We apply our algorithm to natural multiclass learning problems and demonstrate its advantages over previously proposed state-of-the-art methods.

Our algorithm is a generalization of the forward greedy selection approach to sparsity in columns. An alternative approach, which has recently been studied in \cite{QuattoniCaCoDa09, Duchi}, generalizes the $\ell_1$ norm based approach, and relies on mixed-norms. We discuss the advantages of the greedy approach over mixed-norms in \secref{sec:related}.

\subsection{Formal problem statement}

Let $\V$ be the set of objects we would like to classify. For example,
$\V$ can be the set of gray scale images of a certain size. For each
object $\v \in \V$, we have a pool of predefined $d$ features, each of
which is a real number in $[-1,1]$. That is, we can represent each $\v
\in \V$ as a vector of features $\x \in [-1,1]^d$. We note that the
mapping from $\v$ to $\x$ can be non-linear and that $d$ can be very
large. For example, we can define $\x$ so that each element $x_i$
corresponds to some patch, $p \in \{\pm 1\}^{q \times q}$, and a
threshold $\theta$, where $x_i$ equals $1$ if there is a patch of $\v$
whose inner product with $p$ is higher than $\theta$. We discuss some
generic methods for constructing features in \secref{sec:nonLinear}.
From this point onward we assume that $\x$ is given.

The set of possible classes is denoted by $\Y = \{1,\ldots,k\}$. Our
goal is to learn a multiclass predictor, which is a mapping from the
features of an object into $\Y$. We focus on the set of predictors
parametrized by matrices $W \in \reals^{k,d}$ that takes the following
form:
\begin{equation} \label{eqn:hypothesis}
h_W(\x) = \argmax_{y \in \Y} (W\x)_y ~.
\end{equation}
That is, the matrix $W$ maps each $d$-dimensional feature vector into
a $k$-dimensional score vector, and the actual prediction is the index
of the maximal element of the score vector. If the maximizer is not
unique, we break ties arbitrarily.

Recall that our goal is to find a matrix $W$ with few non-zero columns. We denote by $W_{\cdot,i}$ the $i$'th column of $W$ and use the notation
$$
\|W\|_{\infty,0} = |\{ i : \|W_{\cdot,i}\|_\infty > 0\} |
$$
 to denote the number of columns of $W$ which are not identically the zero vector. More generally, given a matrix $W$ and a pair of norms $\|\cdot\|_p,\|\cdot\|_r$ we denote $\|W\|_{p,r} = \left\|(\|W_{\cdot,1}\|_p,\ldots,\|W_{\cdot,d}\|_p)\right\|_r$, that is, we apply the $p$-norm on the columns of $W$ and the $r$-norm on the resulting $d$-dimensional vector.

The $0\!\!-\!\!1$ loss of a multiclass predictor $h_W$ on an example $(\x,y)$ is defined as $\indct{h_W(\x) \neq y}$. That is, the $0\!\!-\!\!1$ loss equals $1$ if $h_W(\x) \neq y$ and $0$ otherwise. Since this loss function is not convex with respect to $W$, we use a surrogate convex loss function  based on the following easy to verify inequalities:
 \begin{align} \nonumber
\boldsymbol{1}[h_W(\x) \neq y] &\le \indct{h_W(\x) \neq y} - (W\x)_y +
(W\x)_{h_W(\x)} \\
&\le \max_{y' \in \Y} \indct{y' \neq y} - (W\x)_y + (W\x)_{y'} \label{eq:svm-loss}\\
&\le \ln \sum_{y' \in \Y} e^{\indct{y' \neq y} - (W\x)_y + (W\x)_{y'} }  \label{eq:log-loss}~.
\end{align}
We use the notation $\ell(W,(\x,y))$ to denote the right-hand side
(\eqref{eq:log-loss}) of the above. The loss given in \eqref{eq:svm-loss} is the multi-class hinge loss
\cite{CrammerSi03a} used in Support-Vector-Machines, whereas
$\ell(W,(\x,y))$ is the result of performing a ``soft-max'' operation:
$\max_x f(x) \le (1/p)\ln \sum_x e^{pf(x)}$, where equality holds
for ${p\rightarrow\infty}$. 

This logistic multiclass loss function $\ell(W,(\x,y))$ has several nice properties --- see for example \cite{Zhang04b}. Besides being a convex upper-bound on the $0\!\!-\!\!1$ loss, it is smooth.
The reason we need the loss function to
be both convex and smooth is as follows. If a function is convex, then
its first order approximation at any point gives us a lower bound on
the function at any other point. When the function is also smooth, the
first order approximation gives us both lower and upper bounds on the
value of the function at any other point\footnote{Smoothness
  guarantees that $|f(x)-f(x')-\nabla f(x')(x-x')|\le \beta\|x-x'\|^2$ for some $\beta$ and all $x,x'$. Therefore one can approximate $f(x)$ by $f(x')+\nabla f(x')(x-x')$ and the approximation error is upper bounded by the difference between $x,x'$.}. ShareBoost uses the gradient
of the loss function at the current solution (i.e. the first order
approximation of the loss) to make a greedy choice of which column to
update. To ensure that this greedy choice indeed yields a significant
improvement we must know that the first order approximation is indeed
close to the actual loss function, and for that we need both lower
and upper bounds on the quality of the first order approximation.

Given a training set $S = (\x_1,y_1),\ldots,(\x_m,y_m)$, the average training loss of a matrix $W$ is:
$
L(W) = \frac{1}{m} \sum_{(\x,y) \in S} \ell(W,(\x,y)) .$
We aim at approximately solving the problem
\begin{align} \label{eqn:optProb}
& \min_{W \in \reals^{k,d}} L(W)  ~~~\textrm{s.t.}~~ \|W\|_{\infty,0} \le s  ~.
\end{align}
That is, find the matrix $W$ with minimal training loss among all matrices with column sparsity of at most $s$, where $s$ is a user-defined parameter. Since $\ell(W,(\x,y))$ is an upper bound on $\indct{h_W(\x) \neq y}$, by minimizing $L(W)$ we also decrease the average $0\!\!-\!\!1$ error of $W$ over the training set. In \secref{sec:analysis} we show that for sparse models, a small training error is likely to yield a small error on unseen examples as well.

Regrettably, the constraint $\|W\|_{\infty,0} \leq s$ in \eqref{eqn:optProb} is non-convex, and
solving the optimization problem in \eqref{eqn:optProb} is NP-hard
\citep{Natarajan95,DavisMaAv97}. To overcome the hardness result, the ShareBoost algorithm will follow the forward greedy selection approach. The algorithm comes with formal generalization and sparsity guarantees (described in Section~\ref{sec:analysis}) that makes ShareBoost an attractive multiclass learning engine due to efficiency (both during training and at test time) and accuracy.

\subsection{Related Work} \label{sec:related}

The centrality of the multiclass learning problem has spurred the development of  various approaches for tackling the task. Perhaps the most straightforward approach is a reduction from multiclass to binary, e.g. the one-vs-rest or all pairs constructions. The more direct approach we choose, in particular, the multiclass predictors of the form given in \eqref{eqn:hypothesis}, has been extensively studied and showed a great success in practice --- see for example \cite{DudaHa73,Vapnik98,CrammerSi03a}.

An alternative construction, abbreviated as the \emph{single-vector} model, shares a single weight vector, for all the classes, paired with class-specific feature mappings. This construction is common in generalized additive
models~\cite{HastieTi95}, multiclass versions of boosting~\cite{FreundSc97,SchapireSi99}, and has been popularized lately due to its role in prediction with structured output where the number of classes is exponentially large (see e.g. \cite{TaskarGuKo03}). While this approach can yield predictors with a rather mild dependency of the required features on $k$ (see for example the analysis in \cite{Zhang04b, TaskarGuKo03, FinkShSiUl06}), it relies on a-priori assumptions on the structure of $\X$ and $\Y$. In contrast, in this paper we tackle general multiclass prediction problems, like object recognition or document classification, where it is not straightforward or even plausible how one would go about to construct a-priori good class specific feature mappings, and therefore the single-vector model is not adequate.

The class of predictors of the form given in \eqref{eqn:hypothesis} can be trained using Frobenius norm regularization (as done by multiclass SVM -- see e.g. \cite{CrammerSi03a}) or using $\ell_1$ regularization over all the entries of $W$. However, as pointed out in \cite{QuattoniCaCoDa09}, these regularizers might yield a matrix with many non-zeros columns, and hence, will lead to a predictor that uses many features.

The alternative approach, and the most relevant to our work,  is the use of  mix-norm regularizations like $\|W\|_{\infty,1}$ or $\|W\|_{2,1}$ \cite{LanckrietCrBaElJo04,TurlachVeWr05,ArgyriouEvPo06,Bach08,QuattoniCaCoDa09,Duchi,HuangZh10}. For example, \cite{Duchi} solves the following problem:
 \begin{align} \label{eqn:Duchi}
& \min_{W \in \reals^{k,d}} L(W)  +\lambda \|W\|_{\infty,1}  ~.
\end{align}
which can be viewed as a convex approximation of our objective
(\eqref{eqn:optProb}). This is advantageous from an
optimization point of view, as one can find the global optimum of a
convex problem, but it remains unclear how well the convex program
approximates the original goal. For example, in
Section~\ref{sec:sharing} we show cases where mix-norm regularization
does not yield sparse solutions while ShareBoost does yield a sparse
solution. Despite the fact that ShareBoost tackles a non-convex
program, and thus limited to local optimum solutions, we prove in
Theorem~\ref{thm:sparse} that under mild conditions ShareBoost is
{\em guaranteed} to find an accurate sparse solution whenever such a
solution exists and that the generalization error is bounded as shown in Theorem~\ref{thm:generalization}.

We note that
several recent papers (e.g. \cite{HuangZh10}) established exact
recovery guarantees for mixed norms, which may seem to be stronger
than our guarantee given in \thmref{thm:sparse}. However,  
the assumptions in \cite{HuangZh10} are much stronger than the
assumptions of \thmref{thm:sparse}. In particular, they
have strong noise assumptions and a group RIP like assumption
(Assumption 4.1-4.3 in their paper). In contrast, we impose no such
restrictions. We would like to stress that in many generic
practical cases, the assumptions of \cite{HuangZh10} will not
hold. For example, when using decision stumps, features will be highly
correlated which will violate Assumption 4.3 of \cite{HuangZh10}. 

Another advantage of ShareBoost is that its only parameter is the
desired number of non-zero columns of $W$. Furthermore, obtaining the
whole-regularization-path of ShareBoost, that is, the curve of
accuracy as a function of sparsity, can be performed by a single run
of ShareBoost, which is much easier than obtaining the whole
regularization path of the convex relaxation in
\eqref{eqn:Duchi}. Last but not least, ShareBoost can work even when
the initial number of features, $d$, is very large, as long as there
is an efficient way to choose the next feature. For example, when the
features are constructed using decision stumps, $d$ will be extremely
large, but ShareBoost can still be implemented efficiently. In
contrast, when $d$ is extremely large mix-norm regularization
techniques yield challenging optimization problems.

As mentioned before, ShareBoost follows the forward greedy selection
approach for tackling the hardness of solving \eqref{eqn:optProb}. The
greedy approach has been widely studied in the context of learning
sparse predictors for linear regression. However, in multiclass
problems, one needs sparsity of groups of variables (columns of $W$).
ShareBoost generalizes the fully corrective greedy selection
procedure given in \cite{ShalevSrZh10} to the case of selection of
groups of variables, and our analysis follows similar techniques.

Obtaining group sparsity by greedy methods has been also recently
studied in \cite{HuangZhMe09,MajumdarWa09}, and indeed, ShareBoost
shares similarities with these works. We differ from
\cite{HuangZhMe09} in that our analysis does not impose strong
assumptions (e.g. group-RIP) and so ShareBoost applies to a much wider
array of applications.  In addition, the specific criterion for
choosing the next feature is different. In \cite{HuangZhMe09}, a ratio
between difference in objective and different in costs is used. In
ShareBoost, the L1 norm of the gradient matrix is used. For the
multiclass problem with log loss, the criterion of ShareBoost is much
easier to compute, especially in large scale
problems. \cite{MajumdarWa09} suggested many other selection rules
that are geared toward the squared loss, which is far from being an
optimal loss function for multiclass problems.

Another related method is the JointBoost algorithm
\cite{TorralbaMuFr06}. While the original presentation in
\cite{TorralbaMuFr06} seems rather different than the type of
predictors we describe in \eqref{eqn:hypothesis}, it is possible to
show that JointBoost in fact learns a matrix $W$ with additional
constraints. In particular, the features $\x$ are assumed to be
decision stumps and each column $W_{\cdot,i}$ is constrained to be
$\alpha_i (\indct{1 \in C_i},\ldots,\indct{k \in C_i})$, where
$\alpha_i \in \reals$ and $C_i \subset \Y$. That is, the stump is
shared by all classes in the subset $C_i$.  JointBoost chooses such
shared decision stumps in a greedy manner by applying the GentleBoost
algorithm on top of this presentation.  A major disadvantage of
JointBoost is that in its pure form, it should exhaustively search $C$
among all $2^k$ possible subsets of $\Y$. In practice,
\cite{TorralbaMuFr06} relies on heuristics for finding $C$ on each
boosting step. In contrast, ShareBoost allows the columns of $W$ to be
any real numbers, thus allowing "soft" sharing between
classes. Therefore, ShareBoost has the same (or even richer)
expressive power comparing to JointBoost. Moreover, ShareBoost
automatically identifies the relatedness between classes
(corresponding to choosing the set $C$) without having to rely on
exhaustive search. ShareBoost is also fully corrective, in the sense
that it extracts all the information from the selected features before
adding new ones. This leads to higher accuracy while using less
features as was shown in our experiments on image classification.
Lastly, ShareBoost comes with theoretical guarantees.

Finally, we mention that feature sharing is merely one way for
transferring information across classes \cite{Thrun96} and several
alternative ways have been proposed in the literature such as target
embedding \cite{HsuKaLaZh09,BengioWeGr11}, shared hidden structure
\cite{LeCunBoBeHa98,AmitFiSrUl07}, shared prototypes
\cite{QuattoniCoDa08}, or sharing underlying metric
\cite{XingNgJoRu03}.

\section{The ShareBoost Algorithm}

ShareBoost is a forward greedy selection approach for solving
\eqref{eqn:optProb}. Usually, in a greedy approach, we update the
weight of one feature at a time. Now, we will update one column of $W$ at a time (since the desired sparsity is over columns). We will choose the column that maximizes the $\ell_1$ norm of the corresponding column of the gradient of the loss at $W$. Since $W$ is a matrix we have that $\nabla L(W)$ is a matrix of the partial derivatives of $L$. Denote
by $\nabla_r L(W)$ the $r$'th column of $\nabla L(W)$, that is, the vector $ \left(\frac{\partial L(W)}{\partial W_{1,r}},\ldots,\frac{\partial L(W)}{\partial W_{k,r}}\right)$. A standard calculation shows that
\begin{align*}
 \frac{\partial L(W)}{\partial W_{q,r}} &= \frac{1}{m} \sum_{(\x,y) \in S} \sum_{c \in \Y} \rho_c(\x,y)\, x_r ( \indct{q=c} - \indct{q=y} )
\end{align*}
where
\begin{equation} \label{eqn:rhoDef}
\rho_c(\x,y) =
\frac{e^{\indct{c \neq y} - (W\x)_{y} + (W\x)_{c} }}{\sum_{y' \in \Y} e^{\indct{y' \neq y} - (W\x)_{y} + (W\x)_{y'} }} .
\end{equation}
Note that $\sum_c \rho_c(\x,y) = 1$ for all $(\x,y)$. Therefore, we can rewrite,
\begin{align*}
\frac{\partial L(W)}{\partial W_{q,r}} &=  \frac{1}{m} \sum_{(\x,y)} x_r  (\rho_q(\x,y) - \indct{q=y}) ~.
\end{align*}
Based on the above we have
\begin{equation} \label{eqn:nablaDef}
\|\nabla_r L(W)\|_1 ~=~ \frac{1}{m} \sum_{q \in \Y} \left| \sum_{(\x,y)} x_r  (\rho_q(\x,y) - \indct{q=y}) \right| ~.
\end{equation}
Finally, after choosing the column for which $\|\nabla_r L(W)\|_1$ is maximized, we re-optimize all the columns of $W$ which were selected so far. The resulting algorithm is given in Algorithm \ref{algo:FGDc}.
\begin{algorithm}[h!]
\caption{ShareBoost}\label{algo:FGDc}
\begin{algorithmic}[1]
\STATE {\bf Initialize:} $W = 0$ ~;~ $I = \emptyset$
\FOR {t=1,2,\ldots,T}
\STATE For each class $c$ and example $(\x,y)$ define $\rho_c(\x,y)$ as in \eqref{eqn:rhoDef}
\STATE Choose feature $r$ that maximizes the right-hand side of \eqref{eqn:nablaDef}
\STATE $I \leftarrow I \cup \{r\}$
\STATE Set $W \leftarrow \argmin_{W} L(W) $ s.t. $W_{\cdot,i} = \boldsymbol{0}$ for all $i \notin I$
\ENDFOR
\end{algorithmic}
\end{algorithm}

The runtime of ShareBoost is as follows. Steps 3-5
requires $O(mdk)$. Step 6 is a convex optimization problem in $tk$
variables and can be performed using various methods. In our
experiments, we used Nesterov's accelerated gradient method
\cite{Nesterov04} 
whose runtime is $O(mtk/\sqrt{\epsilon})$ for a smooth objective, where $\epsilon$ is the desired
accuracy. Therefore, the overall runtime is $O(Tmdk +T^2mk/\sqrt{\epsilon})$. 
It is interesting to compare this runtime to the complexity of minimizing the mixed-norm regularization objective
given in 
\eqref{eqn:Duchi}. Since the objective is
no longer smooth, the runtime of using Nesterov's accelerated method
would be $O(mdk/\epsilon)$ which can be much larger than the runtime
of ShareBoost when $d \gg T$.

\section{Variants of ShareBoost} \label{sec:variants}

We now describe several variants of ShareBoost. The analysis we
present in \secref{sec:analysis} can be easily adapted for these variants as well. 

\subsection{Modifying the Greedy Choice Rule}

ShareBoost chooses the feature $r$ which maximizes the $\ell_1$ norm of the
$r$-th column of the gradient matrix. Our analysis shows that this
choice leads to a sufficient decrease of the objective
function. However, one can easily develop other ways for choosing a
feature which may potentially lead to an even larger decrease of the
objective. For example, we can choose a feature $r$ that minimizes
$L(W)$ over matrices $W$ with support of $I \cup \{r\}$. This will
lead to the maximal possible decrease of the objective function at
the current iteration. Of course, the runtime of choosing $r$ will now be much
larger. Some intermediate options are to choose $r$ that minimizes 

$$\min_{\alpha \in \reals} W + \alpha \nabla_r R(W)$$
 or to choose $r$ that minimizes
$$\min_{\w \in \reals^k} W + \w \e_r^\trans,$$
 where $\e_r^\trans$ is the all-zero
row vector except $1$ in the $r$'th position. 

\subsection{Selecting a Group of Features at a Time} \label{sec:groupVariant}
In some situations, features can be divided into groups where the
runtime of calculating a single feature in each group is almost the
same as the runtime of calculating all features in the group. In such
cases, it makes sense to choose groups of features at each iteration
of ShareBoost. This can be easily done by simply choosing the group of
features $J$ that maximizes $\sum_{j \in J} \|\nabla_j L(W)\|_1$. 

\subsection{Adding Regularization}
Our analysis implies that when $|S|$ is significantly larger than
$\tilde{O}(Tk)$ then ShareBoost will not overfit. When this is not the
case, we can incorporate regularization in the objective of ShareBoost
in order to prevent overfitting. One simple way is to add to the
objective function $L(W)$ a Frobenius norm regularization term of the
form $\lambda \sum_{i,j} W_{i,j}^2$, where $\lambda$ is a
regularization parameter. It is easy to verify that this is a smooth
and convex function and therefore we can easily adapt ShareBoost to
deal with this regularized objective. It is also possible to rely on
other norms such as the $\ell_1$ norm or the $\ell_\infty/\ell_1$
mixed-norm. However, there is one technicality due to the fact that these
norms are not smooth. We can overcome this problem by defining smooth
approximations to these norms. The main idea is to first note that for
a scalar $a$ we have $|a| = \max\{a,-a\}$ and therefore we can rewrite
the aforementioned norms using max and sum operations. Then, we can
replace each max expression with its soft-max counterpart and obtain a
smooth version of the overall norm function. For example, a smooth
version of the $\ell_\infty/\ell_1$ norm will be
$
\|W\|_{\infty,1} \approx \frac{1}{\beta} \sum_{j=1}^d \log\left(\sum_{i=1}^k (e^{\beta W_{i,j}} +
e^{-\beta W_{i,j}})\right) ~,
$
where $\beta \ge 1$ controls the tradeoff between quality of
approximation and smoothness. 

\section{Non-Linear Prediction Rules} \label{sec:nonLinear}
We now demonstrate how ShareBoost can be used for learning non-linear
predictors. The main idea is similar to the approach taken by Boosting
and SVM.  That is, we construct a non-linear predictor by first
mapping the original features into a higher dimensional space and then
learning a linear predictor in that space, which corresponds to a
non-linear predictor over the original feature space. To illustrate
this idea we present two concrete mappings. The first is the decision
stumps method which is widely used by Boosting algorithms. The second
approach shows how to use ShareBoost for learning piece-wise linear
predictors and is inspired by the super-vectors construction recently
described in \cite{ZhouYuZhHu10}.

\subsection{ShareBoost with Decision Stumps}
Let $\v \in \reals^p$ be the original feature vector representing an
object. A decision stump is a binary feature of the form $\indct{v_i
  \le \theta}$, for some feature $i \in \{1,\ldots,p\}$ and threshold
$\theta \in \reals$. To construct a non-linear predictor we can map
each object $\v$ into a feature-vector $\x$ that contains all possible
decision stumps. Naturally, the dimensionality of $\x$ is very large
(in fact, can even be infinite), and calculating Step 4 of ShareBoost
may take forever. Luckily, a simple trick yields an efficient
solution. First note that for each $i$, all stump features
corresponding to $i$ can get at most $m+1$ values on a training set of
size $m$. Therefore, if we sort the values of $v_i$ over the $m$
examples in the training set, we can calculate the value of the
right-hand side of \eqref{eqn:nablaDef} for all possible values of
$\theta$ in total time of $O(m)$. Thus, ShareBoost can be implemented
efficiently with decision stumps.
\begin{figure}
\begin{center}
\includegraphics[width=0.25\textwidth]{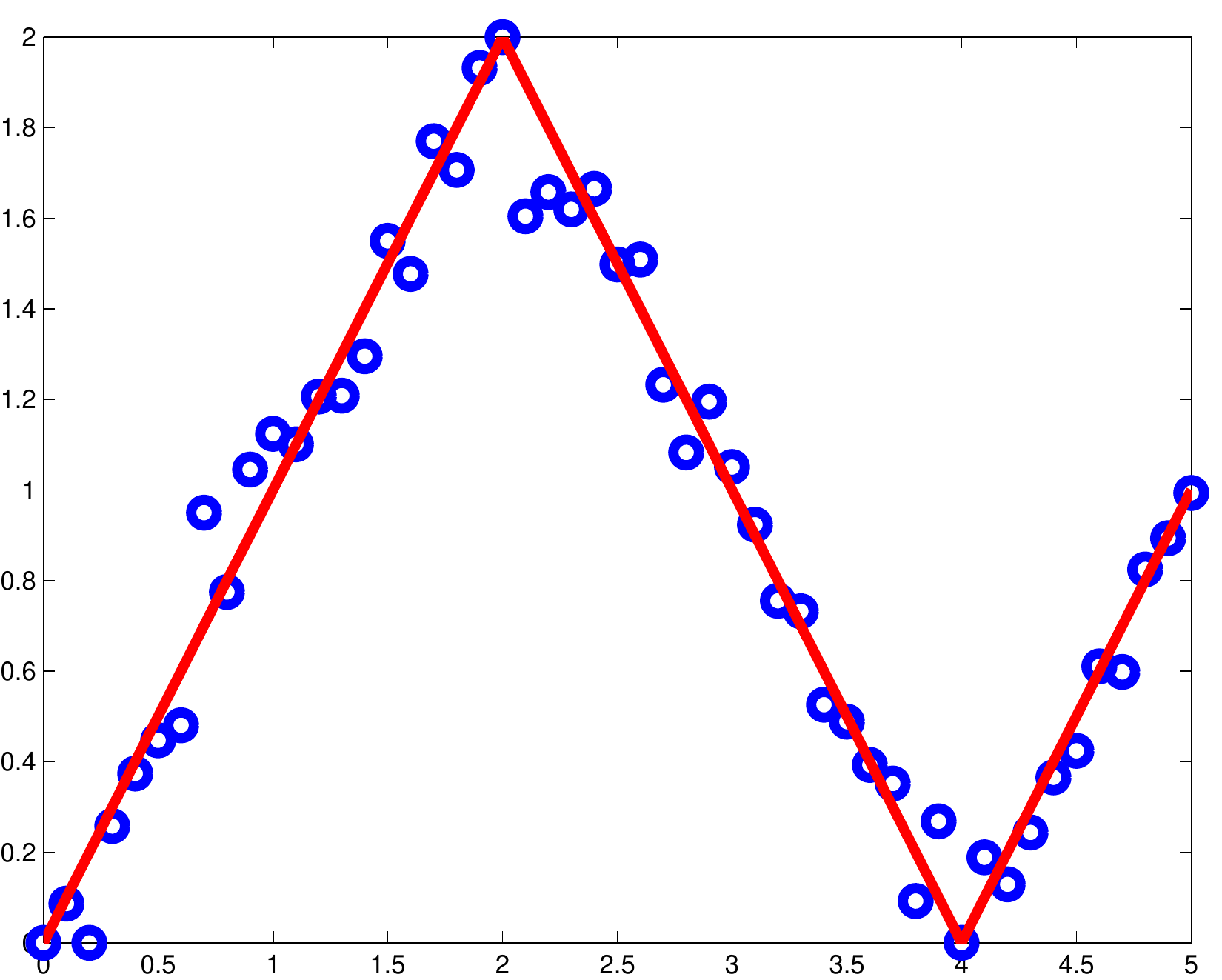}
\end{center}
\caption{{\footnotesize Motivating super vectors.}} \vspace{-.5cm}\label{fig:superVectors}
\end{figure}

\subsection{Learning Piece-wise Linear Predictors with
  ShareBoost}\label{sec:supervector}

To motivate our next construction let us consider first a simple one
dimensional function estimation problem. Given sample
$(x_1,y_i),\ldots,(x_m,y_m)$ we would like to find a function $f :
\reals \to \reals$ such that $f(x_i) \approx y_i$ for all $i$. The
class of piece-wise linear functions can be a good candidate for the
approximation function $f$. See for example an illustration in
\figref{fig:superVectors}. 
In fact, it is easy to verify that all
smooth functions can be approximated by piece-wise linear functions
(see for example the discussion in \cite{ZhouYuZhHu10}).
In general, we can express piece-wise linear vector-valued functions as
$$
f(\v) = \sum_{j=1}^q \indct{\|\v-\v_j\| < r_j} (\inner{\u_j,\v} + b_j) ~,
$$
where $q$ is the number of pieces, $(\u_j,b_j)$ represents the
linear function corresponding to piece $j$, and $(\v_j,r_j)$ represents the
center and radius of piece $j$. This expression can be also written as
a linear function over a different domain, $f(\v) =
\inner{\w,\psi(\v)}$ where
{\small \[
\psi(\v) = \left[ \, \indct{\|\v-\v_1\| < r_1} [ \v \,,\, 1 ] \,,\, \ldots
\,,\,  \indct{\|\v-\v_q\| < r_q} [ \v \,,\, 1 ] \,\right] ~.
\]}
In the case of learning a multiclass predictor, we shall learn a
predictor $\v \mapsto W \psi(\v)$, where $W$ will be a $k$ by
$\textrm{dim}(\psi(\v))$ matrix. ShareBoost can be used for learning
$W$. Furthermore, we can apply the variant of ShareBoost described in
\secref{sec:groupVariant} to learn a piece-wise linear model which few
pieces (that is, each group of features will correspond to one piece
of the model). In practice, we first define a large set of candidate
centers by applying some clustering method to the training examples,
and second we define a set of possible radiuses by taking values of
quantiles from the training examples. Then, we train ShareBoost so as
to choose a multiclass predictor that only use few pairs $(\v_j,r_j)$.

The advantage of using ShareBoost here is that while it learns a
non-linear model it will try to find a model with few linear
``pieces'', which is advantageous both in terms of test runtime as
well as in terms of generalization performance.

\section{Analysis} \label{sec:analysis}

In this section we provide formal guarantees for the ShareBoost algorithm. The proofs are deferred to the appendix.
We first show that if the algorithm has managed to find a matrix $W$ with a small number of non-zero columns and a small training error, then the generalization error of $W$ is also small.
The bound below is in terms of the $0\!\!-\!\!1$ loss. A related bound, which is given in terms of the convex loss function, is described in \cite{Zhang04b}.
\begin{theorem} \label{thm:generalization}
Suppose that the ShareBoost algorithm runs for $T$ iterations and let $W$ be its output matrix.
Then, with probability of at least $1-\delta$ over the choice of the training set $S$ we have that
\bea
&&\prob_{(\x,y) \sim \D}  \left[h_W(\x) \neq y \right] \le
\prob_{(\x,y) \sim S} \left[h_W(\x) \neq y \right]  \\
&&+ O\left( \sqrt{\frac{Tk\log(Tk)\log(k) + T\log(d) + \log(1/\delta)}{|S|} } \right)
\eea
\end{theorem}

Next, we analyze the sparsity guarantees of ShareBoost. As mentioned previously, exactly solving \eqref{eqn:optProb} is known to be NP hard. The following main theorem gives an interesting approximation guarantee. It tells us that if there exists an accurate solution with small $\ell_{\infty,1}$ norm, then the ShareBoost algorithm will find a good sparse solution. 
\begin{theorem} \label{thm:sparse}
Let $\epsilon > 0$ and let $W^\star$ be an arbitrary matrix. Assume that we run the ShareBoost algorithm for $T = \left\lceil  4\,\tfrac{1}{\epsilon}\,\|W^\star\|_{\infty,1}^2 \right\rceil$
iterations and let $W$ be the output matrix. Then, $\|W\|_{\infty,0} \le T$ and $L(W) \le L(W^\star) + \epsilon$.
\end{theorem}

\section{Feature Sharing --- Illustrative Examples}
\label{sec:sharing}

In this section we present illustrative examples, showing that
whenever strong feature sharing is possible then ShareBoost will find
it, while competitive methods might fail to produce solutions with a
small number of features.

In the analysis of the examples below we use the following simple corollary of \thmref{thm:sparse}.
\begin{corollary} \label{cor:shared}
Assume that there exists a matrix $W^\star$ such that $L(W^\star) \le \epsilon$, all entries of $W^\star$ are in $[-c,c]$, and $\|W^\star\|_{\infty,0} = r$. Then, ShareBoost will find a matrix $W$ with $L(W) \le 2\epsilon$ and $\|W\|_{\infty,0} \le 4r^2 c^2/\epsilon$.
\end{corollary}

The first example we present shows an exponential gap between the
number of features required by ShareBoost (as well as mixed norms) and the number of features
required by $\ell_2$ or $\ell_1$ regularization methods.
Consider a set of examples such that each
example, $(\x,y)$, is of the form $\x =
[\textrm{bin}(y) , 2\log(k)\,\e^y] \in \reals^{\log(k)+k}$, where
$\textrm{bin}(y) \in \{\pm 1\}^{\log(k)}$ is the binary representation of the number $y$ in
the alphabet $\{\pm 1\}$ and $\e^y$ is the vector which is zero
everywhere except $1$ in the $y$'th coordinate. For example, if $k=4$
then $\textrm{bin}(1) = [-1, 1]$,  $\textrm{bin}(2) = [1 ,-1]$,
$\textrm{bin}(3) = [1, 1]$, and $\textrm{bin}(4) = [-1 ,-1]$. 

Consider two matrices. The first matrix, denoted $W^{(s)}$, is the
matrix whose row $y$ equals to $[\textrm{bin}(y) , (0,\ldots,0)]$. The
second matrix, denoted $W^{(f)}$, is the matrix whose row $y$ equals
to $[(0,\ldots,0) , \e^y]$. Clearly,  the number of
features used by $h_{W^{(s)}}$ is $\log(k)$ while the number of
features used by  $h_{W^{(f)}}$ is $k$. 

Observe that both $h_{W^{(f)}}(\x)$ and
$h_{W^{(s)}}(\x)$ (see definition in \eqref{eqn:hypothesis}), will
make perfect predictions on the training set. 
Furthermore, since for each example $(\x,y)$, for each $r \neq y$ we
have that $(W^{(s)}\x)_r \in [-\log(k),\log(k)-2]$, for the logistic multiclass loss, for any $c>0$ we
have that 
\begin{align*}
L(cW^{(f)}) &= \log(1 + (k-1)e^{1-2c\log(k)}) \\
&< L(cW^{(s)}) \\
&< \log(1+(k-1)e^{1-c(\log(k)-2)})~.
\end{align*}
It follows that for 
\[ c \ge \frac{1 + \log(k-1)-\log(e^{\epsilon}-1)}{\log(k)-2}
\] we have that  
$L(cW^{(s)}) \le \epsilon$.

Consider an algorithm that solves the regularized problem
\[
\min_W L(W) + \lambda\, \|W\|_{p,p} ~,
\]
where $p$ is either $1$ or $2$. In both cases, we have that\footnote{$\|W^{(f)}\|_{p,p}^p=k$ whereas $\|W^{(s)}\|_{p,p}^p=k\log(k)$.}
$\|W^{(f)}\|_{p,p} < \|W^{(s)}\|_{p,p}$.
It follows that for any value
of $\lambda$, and for any $c>0$, the value of the objective at
$cW^{(f)}$ is smaller than the value at $cW^{(s)}$. In fact, it is not
hard to show that the optimal solution takes the form $c W^{(f)}$ for
some $c>0$. Therefore, no matter what the regularization parameter
$\lambda$ is, the solution of the above regularized problem will use
$k$ features, even though there exists a rather good solution that
relies on $\log(k)$ shared features.

In contrast, using \corref{cor:shared} we know that if we stop
ShareBoost after $\poly(\log(k))$ iterations it will produce a matrix
that uses only $\poly(\log(k))$ features and has a small loss.
Similarly, it is possible to show that for an appropriate
regularization parameter, the mix-norm regularization
$\|W\|_{\infty,1}$ will also yield the matrix $W^{(s)}$ rather than
the matrix $W^{(f)}$.

In our second example we show that in some situations using the
mix-norm regularization,
\[\min_W L(W) + \lambda \|W\|_{\infty,1} ~,
\] will also fail to produce a sparse solution, while ShareBoost is
still guaranteed to learn a sparse solution. Let $s$ be an integer and
consider examples $(\x,y)$ where each $\x$ is composed of $s$ blocks,
each of which is in $\{\pm 1\}^{\log(k)}$. We consider two types
of examples. In the first type, each block of $\x$ equals to
$\textrm{bin}(y)$. In the second type, we generate example as in the
first type, but then we zero one of the blocks (where we choose
uniformly at random which block to zero). As before, $(1-\epsilon)m$
examples are of the first type while $\epsilon m$ examples are of the
second type.

Consider again two matrices. The first matrix, denoted $W^{(s)}$, is
the matrix whose row $y$ equals to $[\textrm{bin}(y) , (0,\ldots,0)]$.
The second matrix, denoted $W^{(f)}$, is the matrix whose row $y$
equals to $[\textrm{bin}(y) , \ldots, \textrm{bin}(y)]/s$.
Note that
$\|W^{(f)}\|_{\infty,1} = \|W^{(s)}\|_{\infty,1}$. In addition, for
any $(\x,y)$ of the second type we have that $\E[W^{(s)}\x] =
W^{(f)}\x$, where expectation is with respect to the choice of which
block to zero. Since the loss function is strictly convex, it follows
from Jensen's inequality that $L(W^{(f)}) < L(W^{(s)})$. We have thus
shown that using the $(\infty,1)$ mix-norm as a regularization will
prefer the matrix $W^{(f)}$ over $W^{(s)}$. In fact, it is possible to
show that the minimizer of $L(W) + \lambda \|W\|_{\infty,1}$ will be
of the form $c W^{(f)}$ for some $c$. Since the number of blocks, $s$,
was arbitrarily large, and since ShareBoost is guaranteed to learn a
matrix with at most $\poly(\log(k))$ non-zero columns, we conclude
that there can be a substantial gap between mix-norm regularization
and ShareBoost. The advantage of ShareBoost in this example follows
from its ability to break ties (even in an arbitrary way). 

Naturally, the aforementioned examples are synthetic and capture
extreme situations. However, in our experiments below we show that
ShareBoost performs better than mixed-norm regularization on natural
data sets as well.

\section{Experiments}

In this section we demonstrate the merits (and pitfalls) of ShareBoost
by comparing it to alternative algorithms in different scenarios.  The
first experiment exemplifies the feature sharing property of
ShareBoost. We perform experiments with an OCR data set and
demonstrate a mild growth of the number of features as the number of
classes grows from 2 to 36.  The second experiment compares
ShareBoost to mixed-norm regularization and to the JointBoost
algorithm of \cite{TorralbaMuFr06}. We follow the same experimental
setup as in \cite{Duchi}. The main finding is that ShareBoost
outperforms the mixed-norm regularization method when the output
predictor needs to be very sparse, while mixed-norm regularization can
be better in the regime of rather dense predictors. We also show that
ShareBoost is both faster and more accurate than JointBoost. The third and final set of experiments is on the MNIST handwritten digit dataset where we demonstrate state-of-the-art accuracy at extremely efficient runtime performance.

\subsection{Feature Sharing}

The main motivation for deriving the ShareBoost algorithm is the need
for a multiclass predictor that uses only few features, and in
particular, the number of features should increase slowly with the
number of classes. To demonstrate this property of ShareBoost we
experimented with the Char74k data set which consists of images of
digits and letters. We trained ShareBoost with the number of classes
varying from 2 classes to the 36 classes corresponding to the 10
digits and 26 capital letters.  We calculated how many features were
required to achieve a certain fixed accuracy as a function of the
number of classes. The description of the
feature space is described in Section~\ref{sec:mnist}. 
\begin{figure}
\begin{center}
\includegraphics[width=0.5\textwidth]{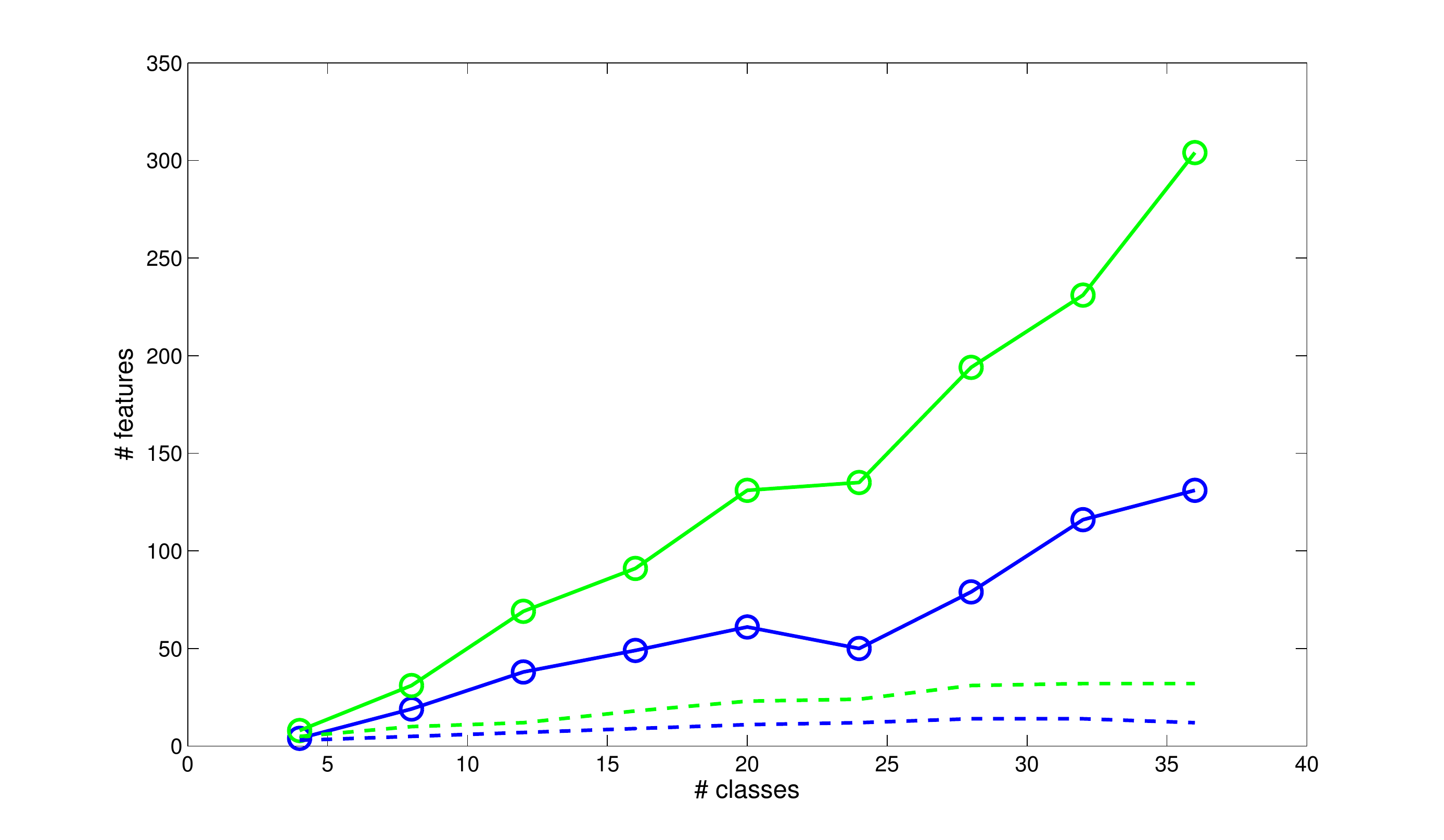}
\end{center}
\caption{{\small The number of features required to achieve a fixed accuracy
  as a function of the number of classes for ShareBoost (dashed) and the
  1-vs-rest (solid-circles). The blue lines are for a target error of
 20\% and the green lines are for 8\%.}} \vspace{-0.5cm}
\label{fig:sharing}
\end{figure}

We compared ShareBoost to the 1-vs-rest approach, where in the latter,
we trained each binary classifier using the same mechanism as used by
ShareBoost. Namely, we minimize the binary logistic loss using a
greedy algorithm. Both methods aim at constructing sparse predictors
using the same greedy approach. The difference between the methods is that
ShareBoost selects features in a shared manner while the 1-vs-rest
approach selects features for each binary problem separately.  In
\figref{fig:sharing} we plot the overall number of features required
by both methods to achieve a fixed accuracy on the test set as a
function of the number of classes. As can be easily seen, the increase
in the number of required features is mild for ShareBoost but
significant for the 1-vs-rest approach.

\subsection{Comparing ShareBoost to Mixed-Norms Regularization} \label{sec:duchi}

Our next experiment compares ShareBoost to the use of mixed-norm
regularization (see \eqref{eqn:Duchi}) as a surrogate for the
non-convex sparsity constraint. See \secref{sec:related} for
description of the approach. To make the comparison fair, we
followed the same experimental setup as in \cite{Duchi} (using code
provided by \author{Duchi}). 

We calculated the whole regularization path for the mixed-norm
regularization by running the algorithm of \cite{Duchi} with many
values of the regularization parameter $\lambda$. In \figref{fig:UCI}
we plot the results on three UCI datasets: StatLog, Pendigits and Isolet. The number
of classes for the datasets are 7,10,26, respectively. The original
dimensionality of these datasets is not very high and therefore,
following \cite{Duchi}, we expanded the features by taking all
products over ordered pairs of features. After this transformation,
the number of features were 630, 120, 190036, respectively. 

Fig.~\ref{fig:UCI} displays the results. As can be seen, ShareBoost
decreases the error much faster than the mixed-norm regularization,
and therefore is preferable when the goal is to have a rather sparse
solution. When more features are allowed, ShareBoost starts to
overfit. This is not surprising since here sparsity is our only mean
for controlling the complexity of the learned classifier. To prevent
this overfitting effect, one can use the variant of ShareBoost that
incorporates regularization---see \secref{sec:variants}. 

\begin{figure*}
\centering
\begin{tabular}{ccc}
\includegraphics[width=0.3\textwidth]{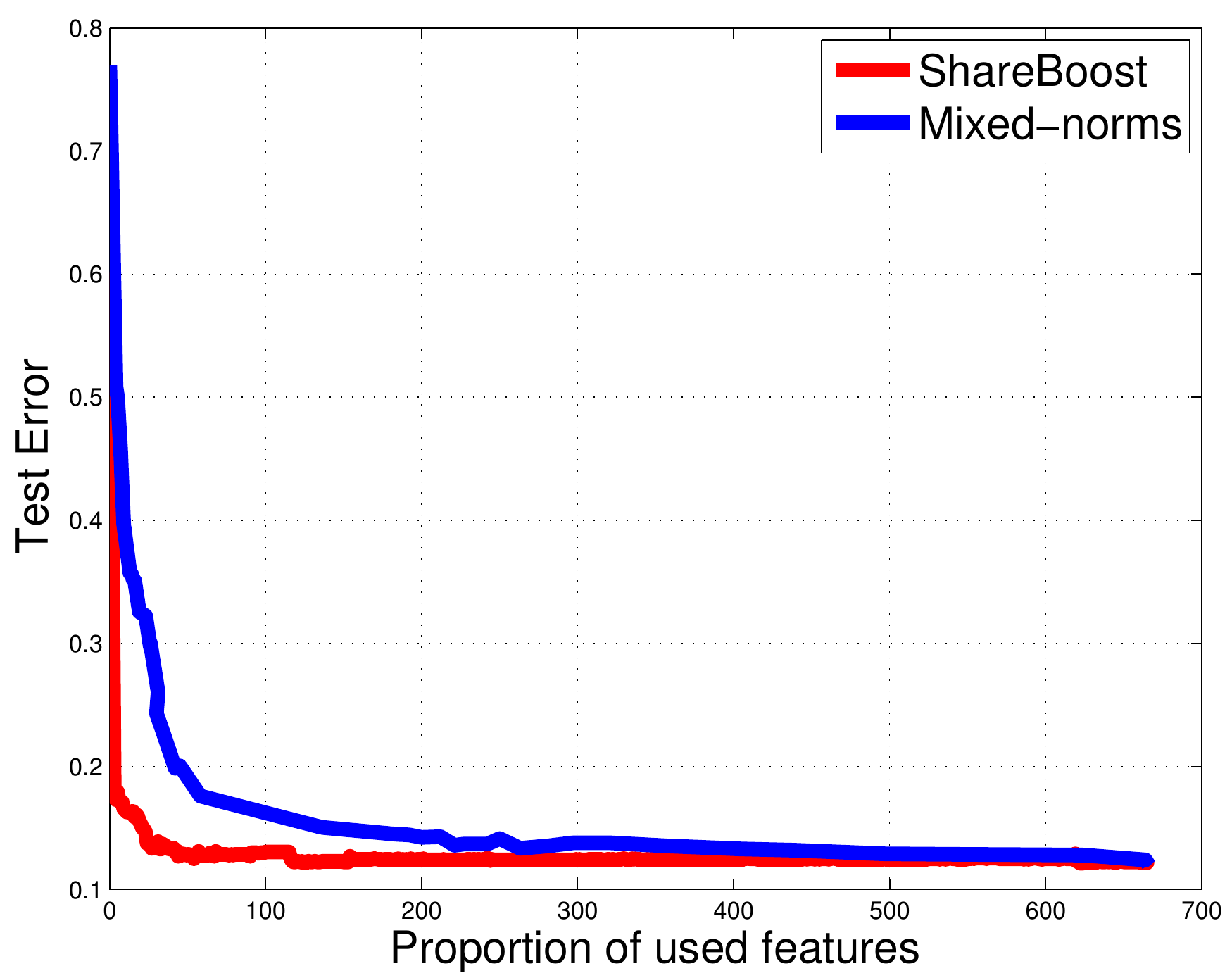}&
\includegraphics[width=0.3\textwidth]{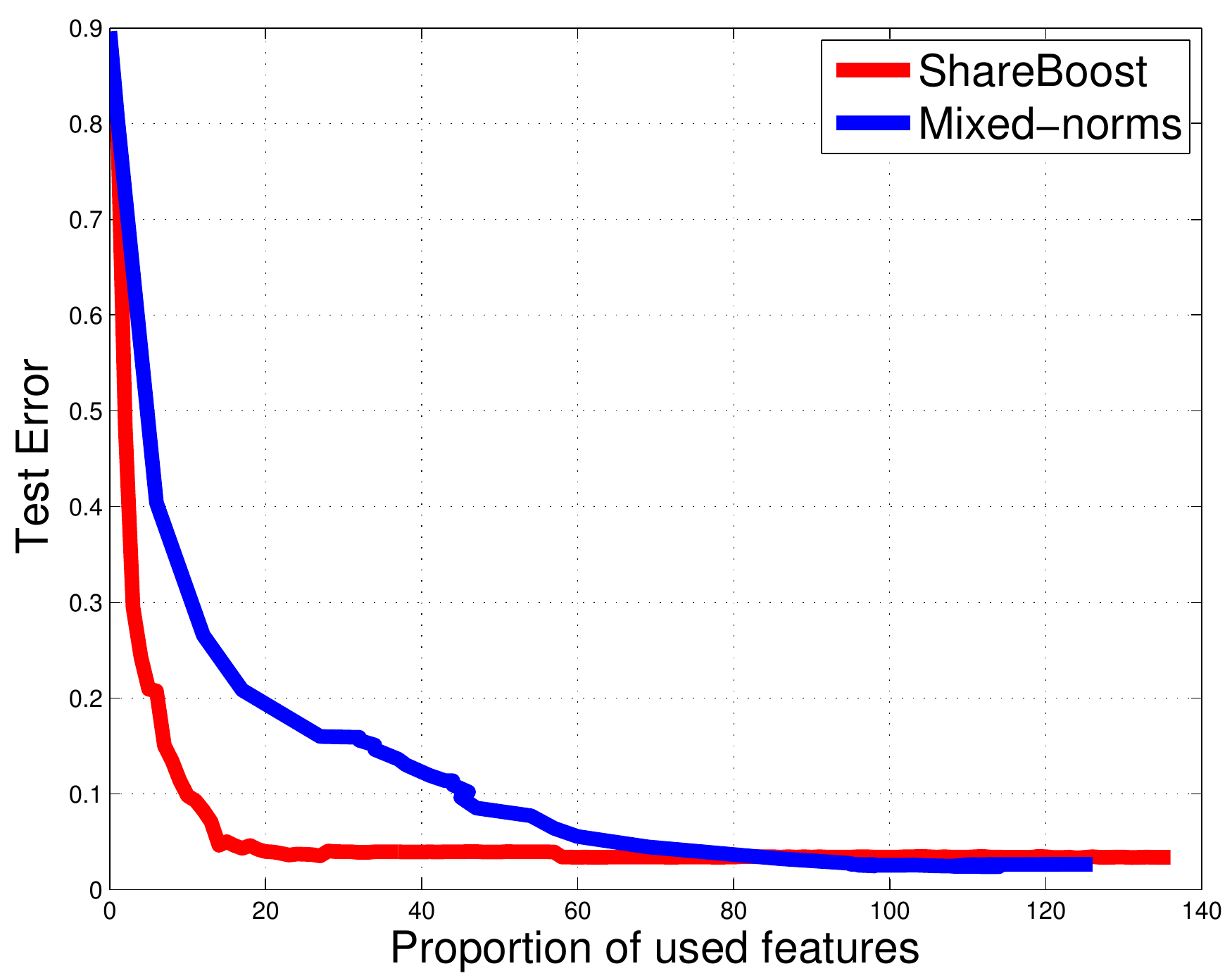} &
\includegraphics[width=0.3\textwidth]{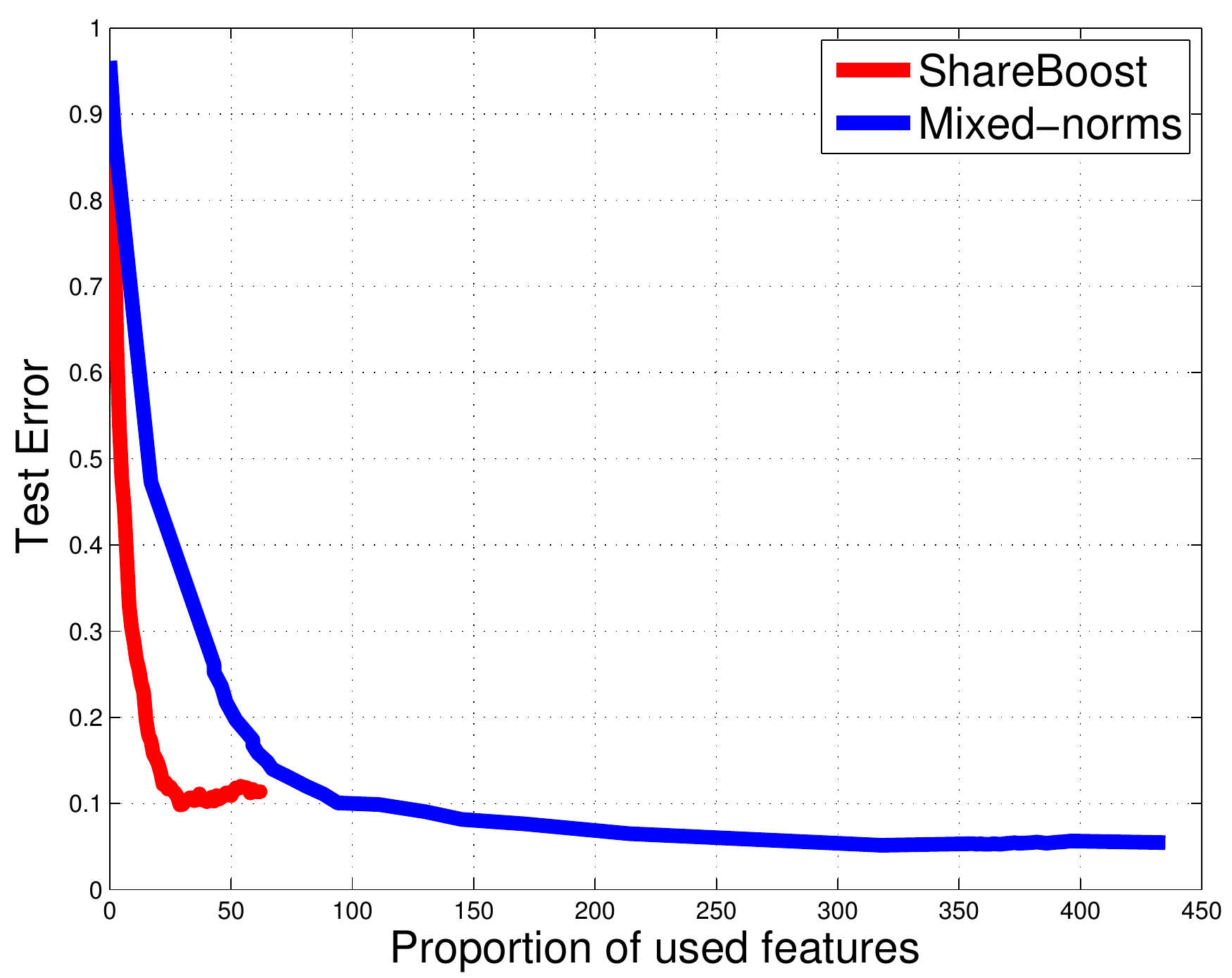}\\
\includegraphics[width=0.3\textwidth]{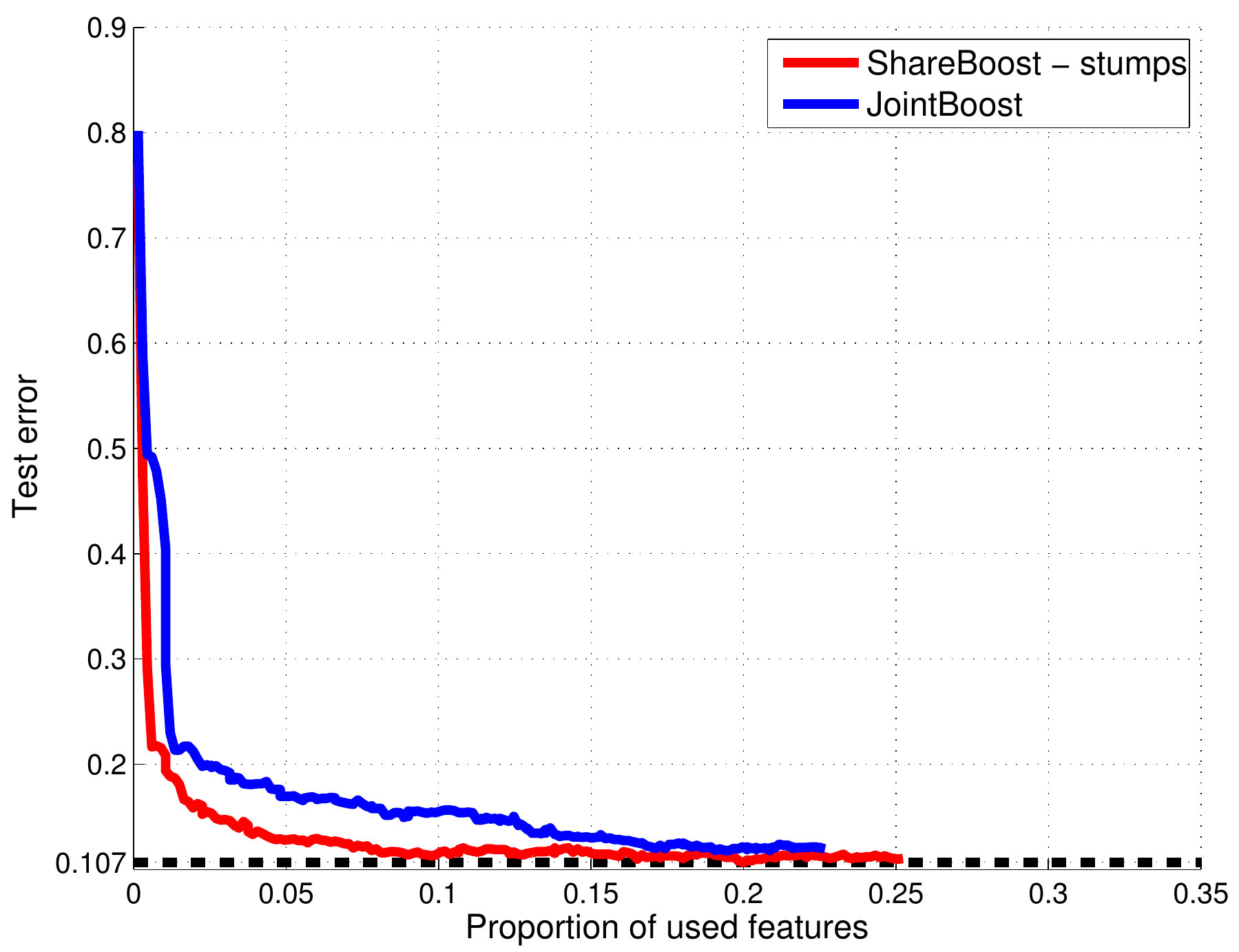}&
\includegraphics[width=0.3\textwidth]{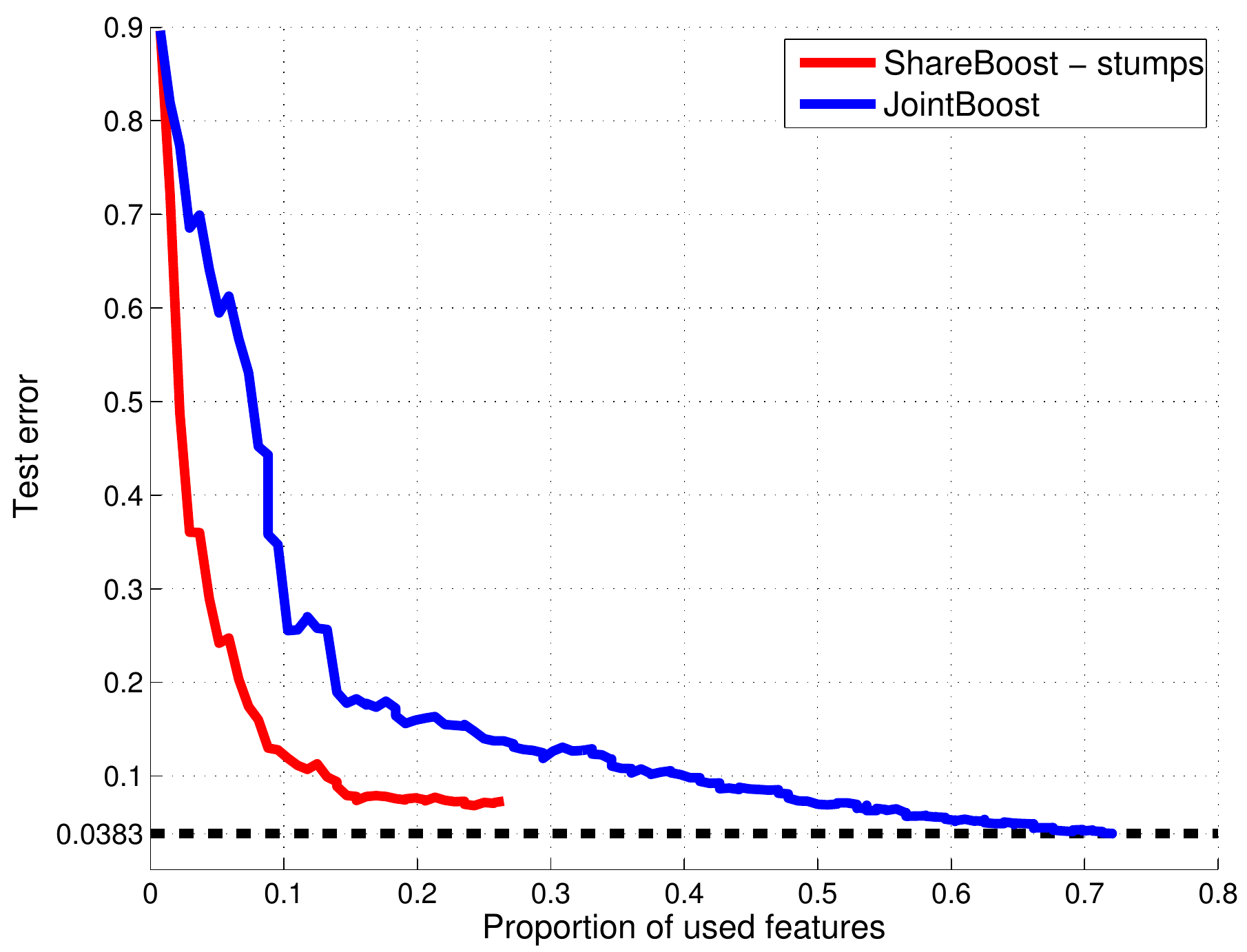} &
\includegraphics[width=0.3\textwidth]{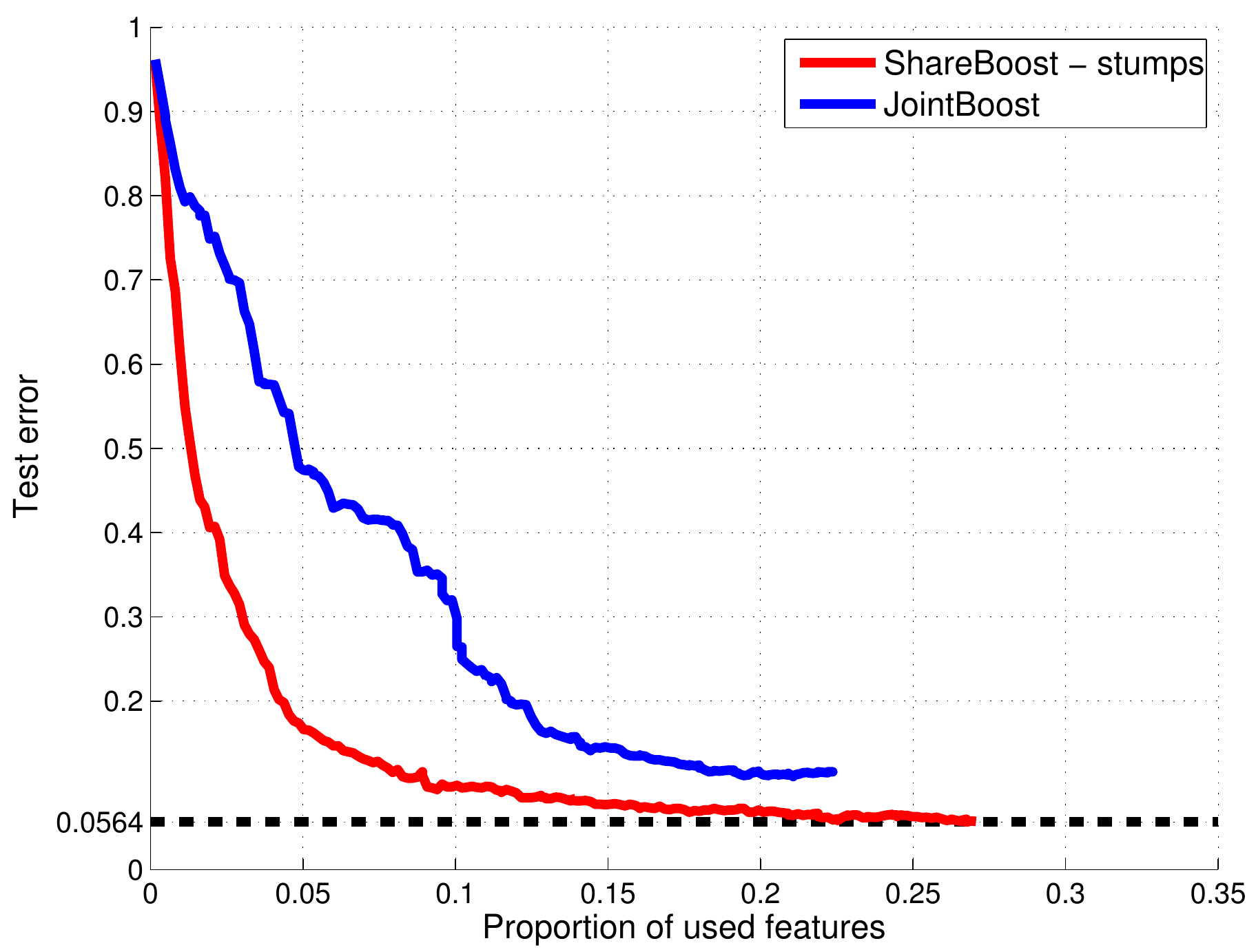}\\
(a) StatLog & (b) Pendigits & (c) Isolet
\end{tabular}
\caption{\small ShareBoost compared with mixed-norm regularization
  (top) and JointBoost (bottom) on several UCI
  datasets.  The horizontal axis is the feature sparsity (fraction of features
 used) and the vertical axis is the test error rate.}
\label{fig:UCI}
\end{figure*}

\subsection{Comparing ShareBoost to JointBoost} \label{sec:jointBoost}

Here we compare ShareBoost to the JointBoost algorithm of
\cite{TorralbaMuFr06}. See \secref{sec:related} for description of
JointBoost. As in the previous experiment, we followed the
experimental setup as in \cite{Duchi} and ran JointBoost of
\cite{TorralbaMuFr06} using their published code with additional
implementation of the BFS heuristic for pruning the $2^k$ space of all
class-subsets as described in their paper.

Fig.~\ref{fig:UCI} (bottom) displays the results. Here we used stump
features for both algorithms since these are needed for JointBoost.
As can be seen, ShareBoost decreases the error much faster and
therefore is preferable when the goal is to have a rather sparse
solution. As in the previous experiment we observe that when more
features are allowed, ShareBoost starts to overfit. Again, this is not
surprising and can be prevented by adding additional regularization.
The training runtime of ShareBoost is also much shorter than that of
JointBoost (see discussion in \secref{sec:related}).

\begin{table*}
\centering

  \begin{tabular}{|l|c|c|c|c|c|}
\hline
    Reference   & 3NN & Shape Context & SVM 9-poly & Neural Net & ShareBoost \\
                &         & Belongie-et-al & DeCosta-et-al & Ciresan-et-al & \\
\hline
    Error rate  & 2.7\% & 0.63\%  & 0.56\%      & 0.35\%        & 0.47\% \\
\hline
    Errors      & 270   &  63     & 56       &   35          & 47 \\
\hline
    Year        &  --   & 2002      & 2002    & 2010          & 2011 \\
\hline
    Run time    & $\times$ 14             & $\times$ 1000's & $\times$ 38 & $\times$ 2.5 & 1 \\
\hline
  \end{tabular}

\caption{\small Comparison of ShareBoost and relevant methods on error rate and computational complexity over the MNIST dataset. More details in the text. }
\label{tab:mnist-comparison}
\end{table*}

\subsection{MNIST Handwritten Digits Dataset}
\label{sec:mnist}

The goal of this experiment is to show that ShareBoost achieves
state-of-the-art performance while constructing very fast predictors.
We experimented with the MNIST digit dataset, which consists of a training set of $60,000$ digits
represented by centered size-normalized $28\times 28$ images, and a
test set of $10,000$ digits (see Fig.~\ref{fig:mnist-miss} for some
examples).  The MNIST dataset has been extensively studied and is
considered the standard test for multiclass classification of
handwritten digits. The error rate achieved by the most advanced
algorithms are below $1\%$ of the test set (i.e., below $100$
classification mistakes on the test set). To get a sense of the
challenge involved with the MNIST dataset, consider a straightforward
3-Nearest-Neighbor (3NN) approach where each test example $\x$,
represented as a vector with $28^2$ entries, is matched against the
entire training set $\x_j$ using the distance
$d(\x,\x_j)=\|\x-\x_j\|^2$. The classification decision is then
the majority class label of the three most nearest training
examples. This naive 3NN approach achieves an error rate of $2.67\%$
(i.e., $267$ mis-classification errors) with a run-time of unwieldy
proportions. Going from 3NN to qNN with $q=4,...,12$ does not produce
a better error rate.

More advanced shape-similarity measures could improve the performance of the naive $qNN$ approach but at a heavier run-time cost. For example, the Shape Context similarity measure introduced by
\citep{Belongie2002} uses 
a Bipartite matching algorithm between descriptors computed along
 $100$ points in each image. A 3NN using Shape-Context similarity
 achieves an error rate of $0.63\%$ but at a very high (practically
 unwieldy) run-time cost. The challenge with the MNIST dataset is,
 therefore,  to design a multiclass algorithm with a small error rate
 (say below $1\%$) {\it and\/} have an efficient run-time performance.

The top MNIST performer \cite{Ciresan2010} uses a feed-forward Neural-Net with $7.6$ million connections which roughly translates to $7.6$ million multiply-accumulate (MAC) operations at run-time as well. During training, geometrically distorted versions of the original examples were generated  in order to expand the training set following \cite{Simard} who introduced a warping scheme for that purpose. The top performance error rate stands at $0.35\%$ at a run-time cost of $7.6$ million MAC per test example.

Table ~\ref{tab:mnist-comparison} summarizes the discussion so far including the performance of ShareBoost. The error-rate of ShareBoost with $266$ rounds stands on $0.71\%$ using the original training set and $0.47\%$ with the expanded training set of $360,000$ examples generated by adding five deformed instances per original example and with $T=305$ rounds. The run-time on test examples is around $40\%$ of the leading MNIST performer. The error rate of $0.47\%$ is better than  that reported by \cite{Decoste2002} who used a 1-vs-all SVM with a 9-degree polynomial kernel and with an expanded training set of $780,000$ examples. The number of support vectors (accumulated over the ten separate binary classifiers) was $163,410$ giving rise to a run-time of $21$-fold compared to ShareBoost.
We describe below the details of the ShareBoost implementation on the MNIST dataset.

The feature space we designed consists of 7$\times$7 image patches
with corresponding spatial masks, constructed as follows.  All
7$\times$7 patches were collected from all images and clustered using
K-means to produce $1000$ centers $w_f$. For each such center (patch) we also associated a set of $16$ possible masks $g_f$ in order to limit the spatial locations of the maximal response of the $7\times 7$ patch.  The pairs
$F=\{(v_f, g_f)\}$ form the pool of $d=16,000$ templates (shape plus
location).
 The vector of feature measurements $\x\in
R^m=\left(\ldots, x_{fc}, \ldots\right)$ has each of its entries
associated with one of the templates where an entry $x_{fc}=
\max\left\{(I\otimes w_f)\times g_f^c\right\}$.  That is, a feature is
the maximal response of the convolution of the template $w_f$ over the
image, weighted by the Gaussian $g_f^c$.

ShareBoost selects a subset of the templates $j_1, \ldots, j_T$ where
each $j_i$ represents some template pair $(w_{f_i}, g_{f_i}^{c_i})$,
and the matrix $W\in R^{k\times T}$.  A test image $I$ is then
converted to $\tilde \x\in R^T$ using $\tilde x_i = \max\{(I\otimes
w_{f_i})\times g_{f_i}^{c_i}\}$ with the maximum going over the image
locations. The prediction $\hat y$ is then $\argmax_{y \in [k]}
(W\tilde\x)_y$.  Fig.~\ref{fig:mnist-features}(a) shows the first 30
templates that were chosen by ShareBoost and their corresponding
spatial masks.  For example, the first templates matches a digit part
along the top of the image, the eleventh template matches a horizontal
stroke near the top of the image and so forth.  Fig.~\ref{fig:mnist-features}(b)
shows the weights (columns of $W$) of the first 30 templates of the
model that produced the best results. For example, the eleventh
template which encodes a horizontal line close to the top is expected
in the digit ``9'' but not in the digit
``4''. Fig.~\ref{fig:mnist-miss} shows the 47 misclassified samples
after $T=305$ rounds of ShareBoost, and Fig.~\ref{fig:mnist-conv} displays the convergence curve of error-rate as a function of the number of rounds.

In terms of run-time on a test image, the system requires $305$
convolutions of $7\times 7$ templates and $540$ dot-product operations
which totals to roughly $3.3\cdot 10^6$ MAC operations --- compared to
around $7.5\cdot 10^6$ MAC operations of the top MNIST performer.
Moreover, due to the fast convergence of ShareBoost, 75 rounds are
enough for achieving less than $1\%$ error.  Further improvements of
ShareBoost on the MNIST dataset are possible such as by extending
further the training set using more deformations and by increasing the
pool of features with other type of descriptors -- but those were not
pursued here. The point we desired to make is that ShareBoost can
achieve competitive performance with the top MNIST performers, both in
accuracy and in run-time, with little effort in feature space design
while exhibiting great efficiency during training time as well.

\begin{figure}
\centering
\begin{tabular}{cc}
\includegraphics[width=0.22\textwidth]{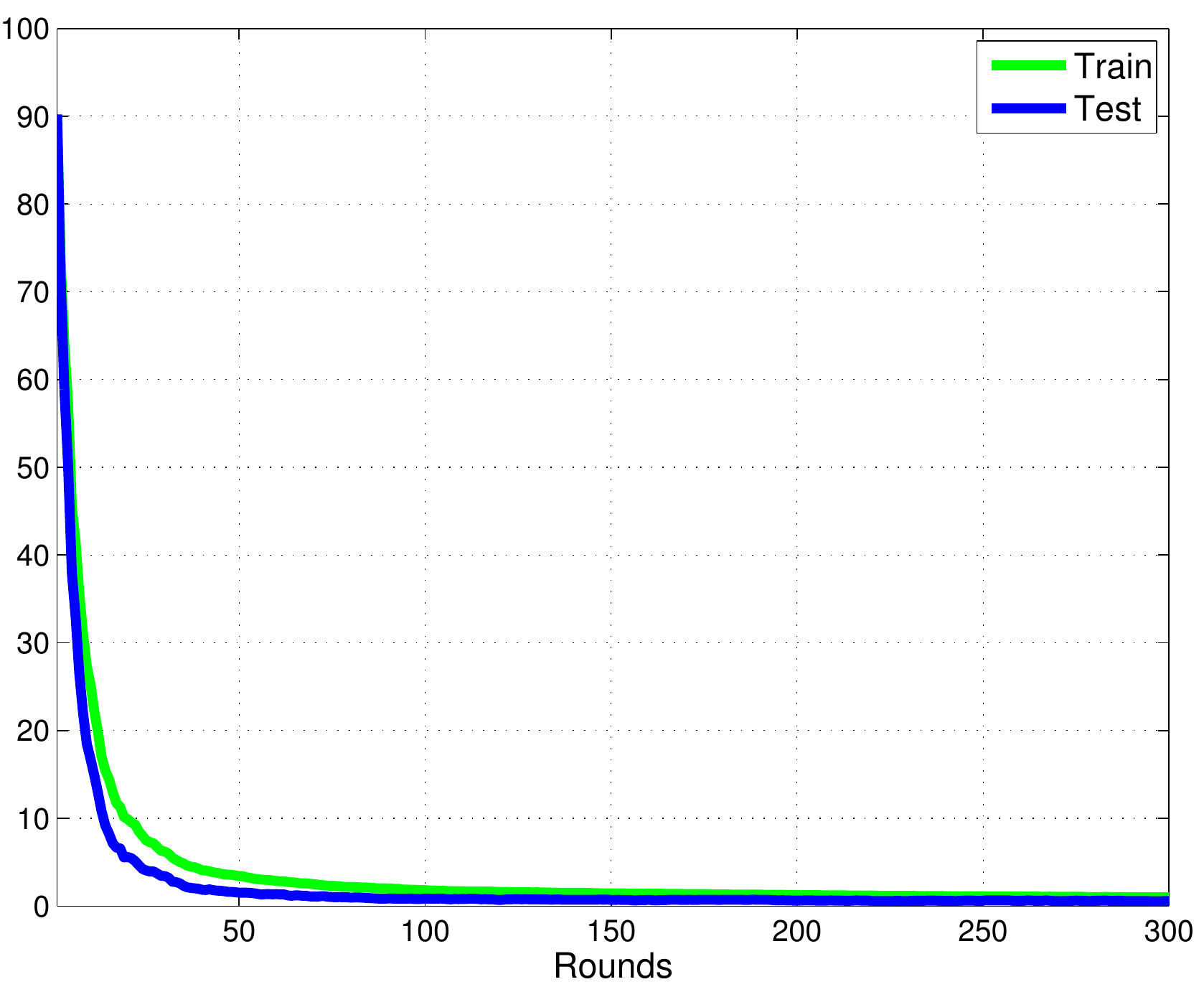}&
\includegraphics[width=0.22\textwidth]{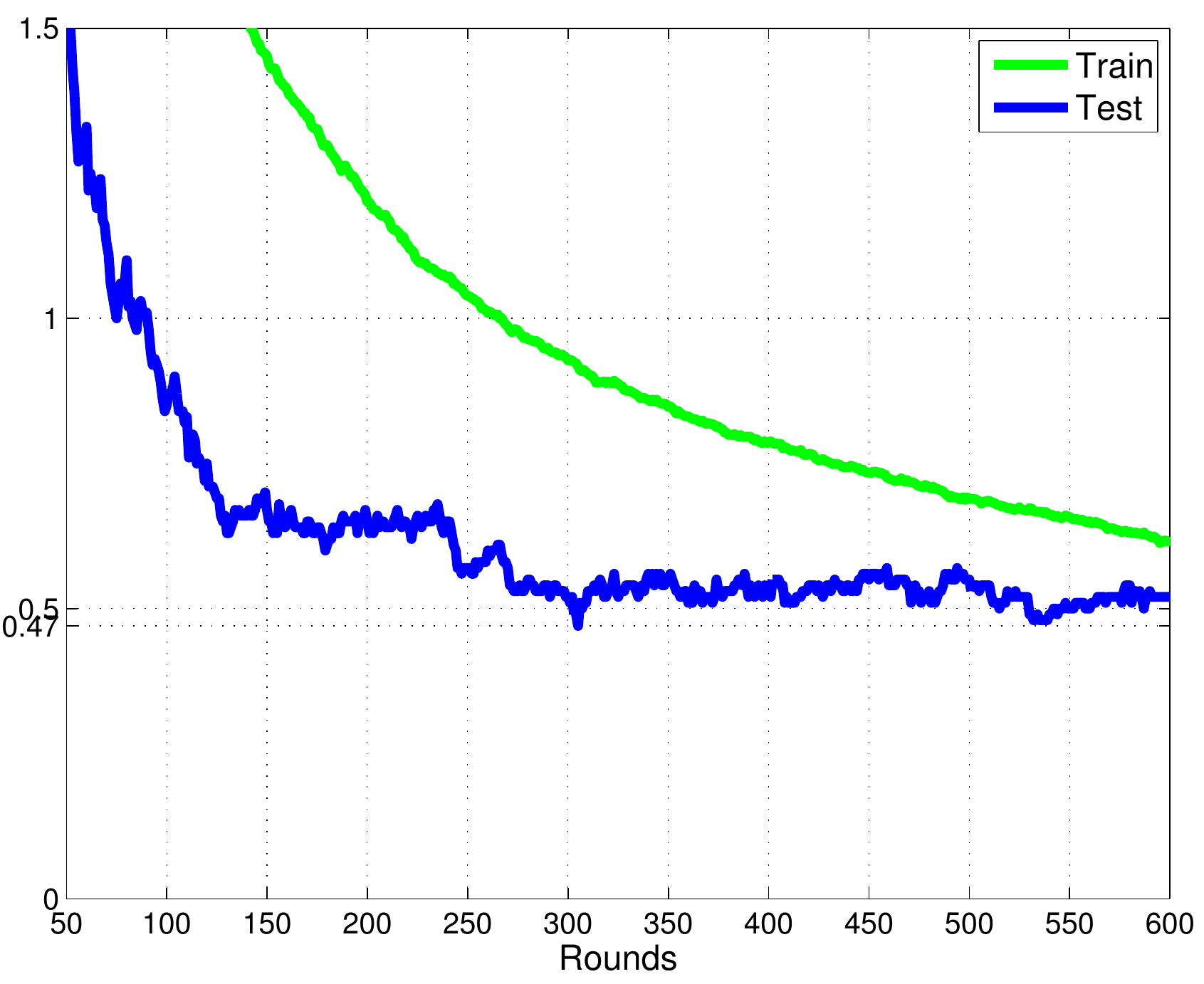}
\end{tabular}
\caption{\small Convergence of Shareboost on the MNIST dataset as it
  reaches 47 errors.  The set was expanded with 5 deformed versions of
  each input, using the method in \cite{Simard}. Since the
  deformations are fairly strong, the training error is higher than
  the test.  Zoomed in version shown on the right.  }
\label{fig:mnist-conv}
\end{figure}

\begin{figure*}
\centering
\begin{tabular}{cc}
\includegraphics[width=0.45\textwidth]{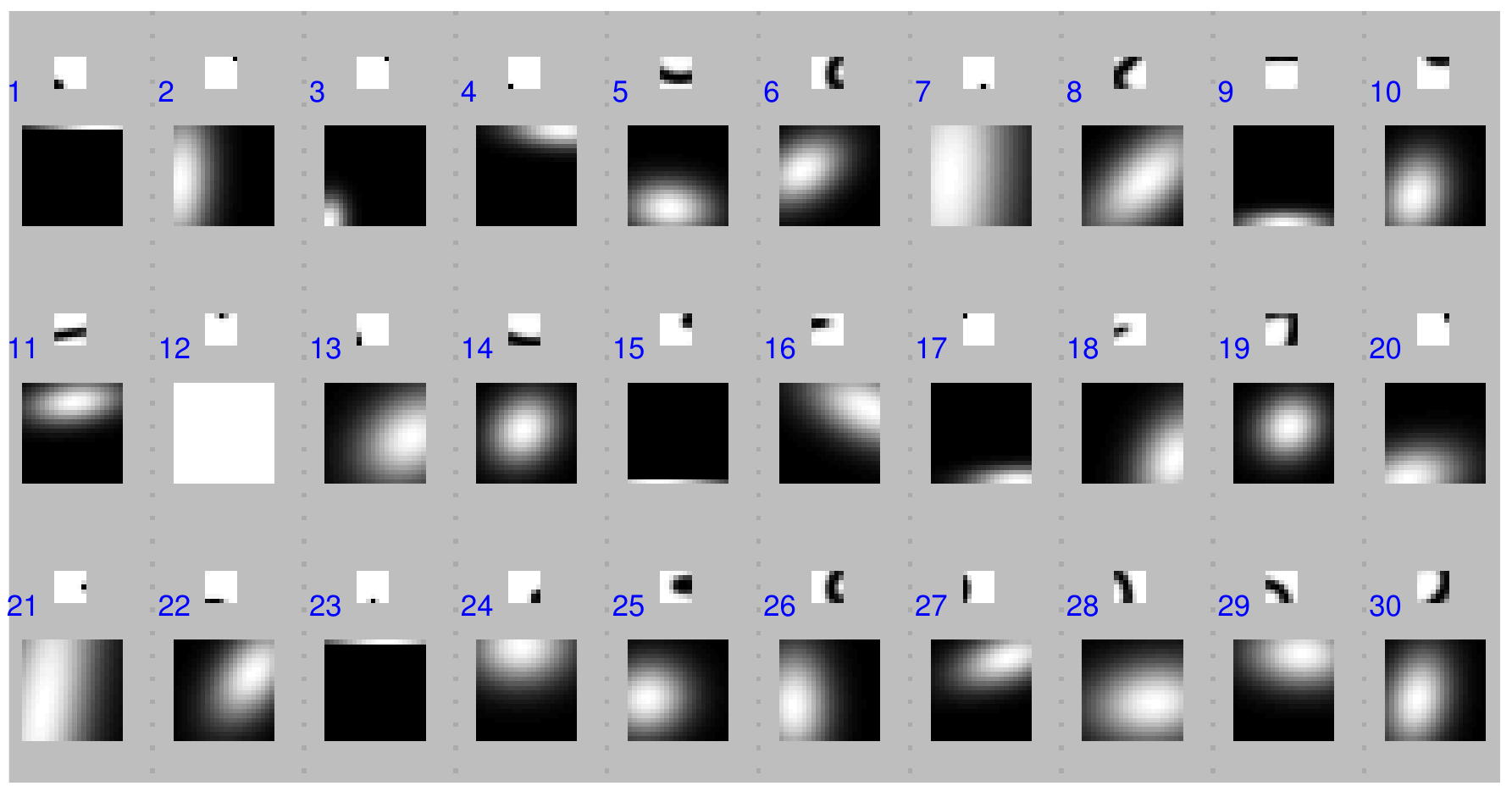}&
\includegraphics[width=0.45\textwidth]{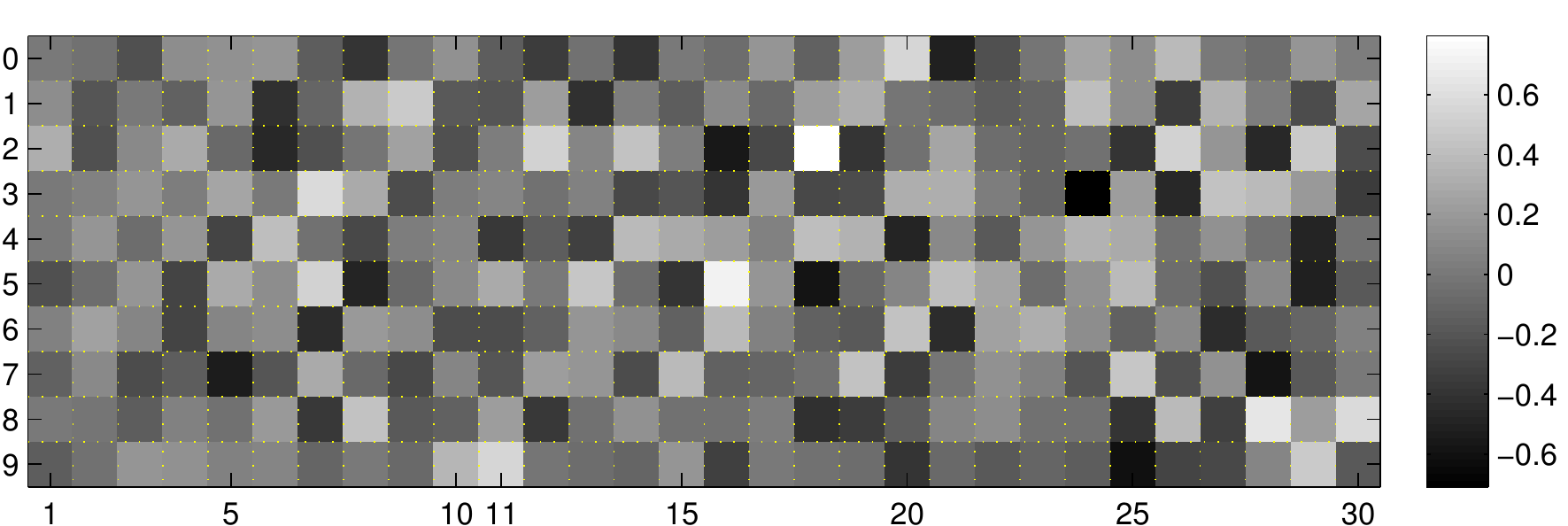}\\
(a) Leading 30 selected features & (b) Corresponding columns of $W$
\end{tabular}
\caption{\small (a) The first 30 selected features for the MNIST
  dataset. Each feature is composed of a 7$\times$7 template and a position
  mask. (b) The corresponding columns of $W$. The entries of a column represents the "sharing" among classes pattern. For example, the eleventh
template which encodes a horizontal line close to the top is expected
in the digits ``9,8,5'' but not in digit
``4''.}
\label{fig:mnist-features}
\end{figure*}

\begin{figure}
\centering
\includegraphics[width=0.45\textwidth]{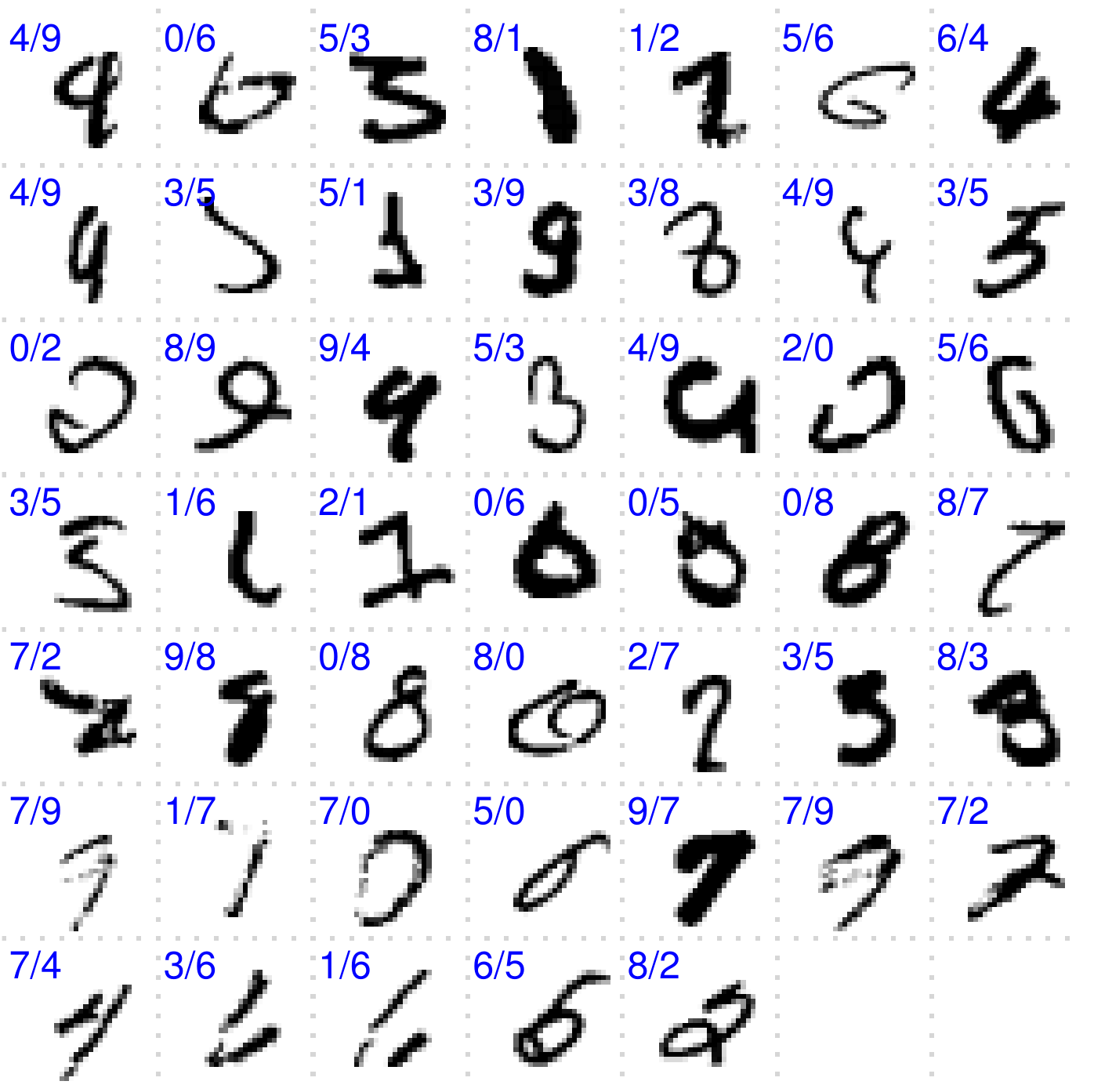}
\caption{\small ShareBoost achieves an error of $0.47\%$ on the test set which translates to $47$ mistakes displayed above. Each error test example is displayed together with its predicted and True labels.}
\label{fig:mnist-miss}
\end{figure}

\subsection{Comparing ShareBoost to kernel-based SVM} \label{sec:mnistPCA}

In the experiments on the MNIST data set reported above,  each
feature is the maximal response of the convolution of a $7 \times 7$ patch
over the image, weighted by a spatial mask.

One might wonder if the stellar performance of ShareBoost is maybe due
to the patch-based features we designed. In this section we remove
doubt by using
ShareBoost for training a piece-wise linear predictor, as described in
Section~\ref{sec:supervector}, on MNIST using generic features.  We
show that ShareBoost comes close to the error rate of SVM with
Gaussian kernels, while only requiring $230$ anchor points, which is
well below the number of support-vectors needed by kernel-SVM. This
underscores the point that ShareBoost can find an extremely fast
predictor without sacrificing state-of-the-art performance level.  

\begin{figure}
\includegraphics[width=0.5\textwidth]{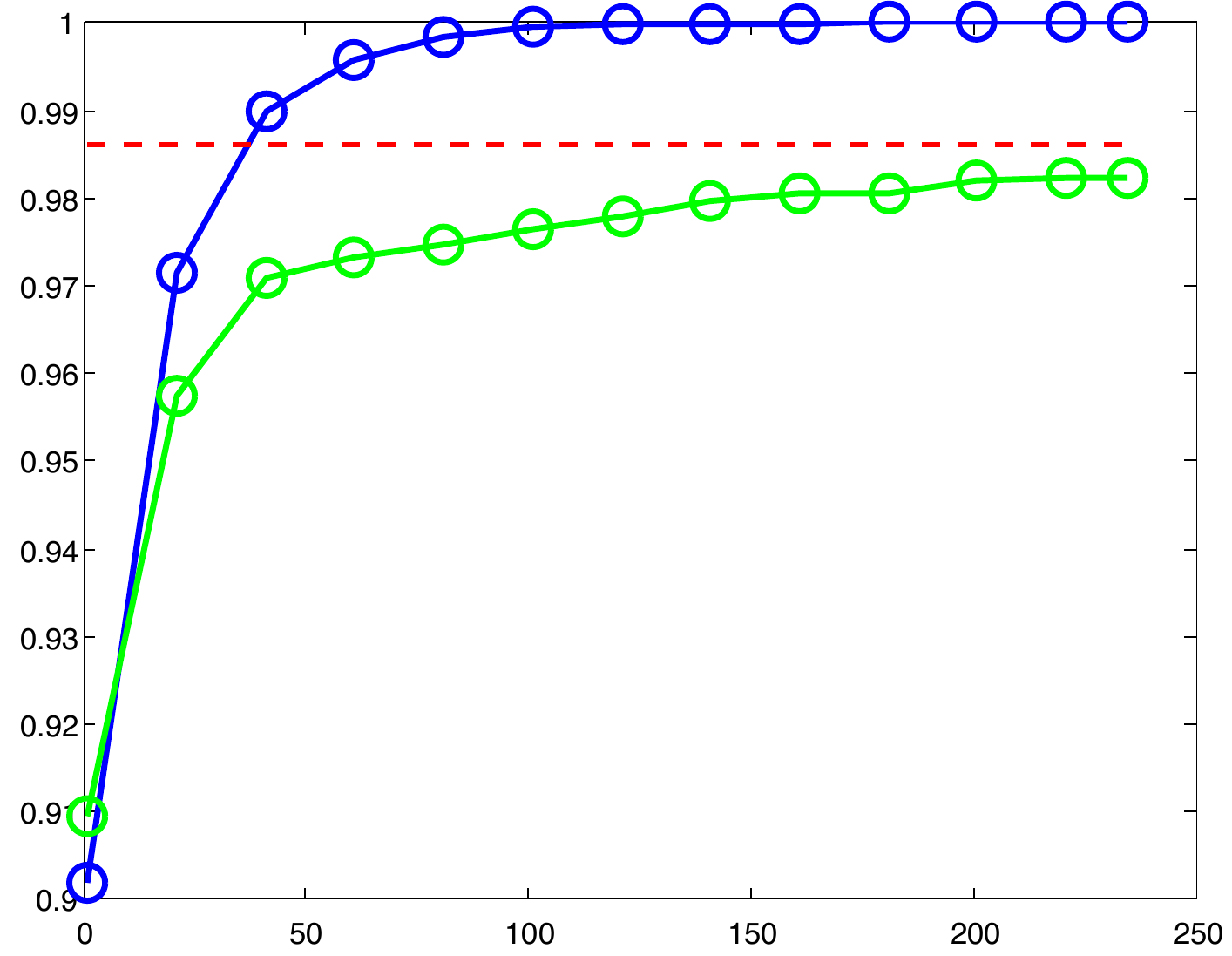}
\caption{\small Test accuracy of ShareBoost on the MNIST
  dataset as a function of the number of rounds using the generic
  piece-wise linear construction. Blue: train accuracy. Red: test
  accuracy. Dashed: SVM with Gaussian kernel accuracy.} \vspace{-0.5cm}
\label{fig:mnist-super}
\end{figure}

Recall that the piece-wise linear predictor is of the following
form: 
{\small 
\[h(\x)= \argmax_{y \in \Y} \left( \sum_{j\in {\cal I}}
  \indct{\|\x-\v^{(j)}\| < r^{(j)}}\,(W_{y,\cdot}^{(j)}\x +
b_{y}^{(j)})\right),\]} 
where $\v^{(j)}\in \reals^d$ are anchor points
with radius of influence $r^{(j)}$, and $W^{(j)},b^{(j)}$ define
together a linear classifier for the $j$'th anchor. ShareBoost selects
the set of anchor points and their radiuses together with the
corresponding linear classifiers.  In this context it is worthwhile to
compare classification performance to SVM with Gaussian kernels applied in
a 1-vs-all framework. Kernel-SVM also selects a subset of the training
set $S$ with corresponding weight coefficients, thus from a
mechanistic point of view our piece-wise linear predictor shares the
same principles as kernel-SVM.

We performed a standard dimensionality reduction using PCA from the
original raw pixel dimension of $28^2$ to $50$, i.e., every digit was
mapped to $\x\in R^{50}$ using PCA. The pool of anchor points was
taken from a reduced training set by means of clustering $S$ into 1500
clusters and the range of radius values per anchor point was taken
from a discrete set of 35 values. Taken together, each round of
ShareBoost selected an anchor point $\v^{(j)}$ and radius $r^{(j)}$
from a search space of size $52500$. Fig.~\ref{fig:mnist-super} shows
the error-rate per ShareBoost rounds. As can be seen, ShareBoost comes
close to the error rate of SVM while only requiring $230$ anchor
points, which is well below the number of support-vectors needed by
kernel-SVM. This underscores the point that ShareBoost can find an
extremely fast predictor without sacrificing state-of-the-art
performance level.

\section{Acknowledgements}

We would like to thank Itay Erlich and Zohar Bar-Yehuda for their dedicated contribution to the implementation of ShareBoost.

\bibliography{curRefs}
\bibliographystyle{icml2011}

\newpage
\onecolumn

\appendix

\section{Proofs}

\subsection{Proof of \thmref{thm:generalization}}
The proof is based on an analysis of the Natarajan dimension of the
class of matrices with small number of non-zero columns. The Natarajan
dimension is a generalization of the VC dimension for classes of
multiclass hypotheses. In particular, we rely on the analysis given in
Theorem 25 and Equation 6 of \cite{DanielySaBeSh11}. This implies that
if the set of $T$ columns of $W$ are chosen in advance then
\begin{align*}
\prob_{(\x,y) \sim \D} & \left[h_W(\x) \neq y \right] ~\le~
\prob_{(\x,y) \sim S} \left[h_W(\x) \neq y \right]  + 
O\left( \sqrt{Tk\log(Tk)\log(k) + \log(1/\delta)}/\sqrt{|S|}  \right) ~.
\end{align*}
Applying the union bound over all $T \choose d$ options to choose the
relevant features we conclude our proof.

\subsection{Proof of \thmref{thm:sparse}}
To prove the theorem, we start by establishing a certain smoothness property of $L$. First, we need the following.
\begin{lemma} \label{lem:smooth1}.
 Let $\ell : \reals^k \to \reals$ be defined as
 \[
\ell(\v) = \log\left(1 + \sum_{i \in [k] \setminus \{j\}} e^{1-v_j+v_i}\right) ~.
 \] Then, for any $\u,\v$ we have
 \[
\ell(\u+\v) \le \ell(\u) + \inner{\nabla \ell(\u),\v} +  \|\v\|_\infty^2 ~.
 \]
\end{lemma}
\begin{proof}
Using Taylor's theorem, it suffices to show that the Hessian of $\ell$ at any point satisfies
\[
\v^\trans H \v \le 2\|\v\|_\infty^2 ~.
\]
Consider some vector $\w$ and without loss of generality assume that $j=1$.
We have,
\[
\frac{\partial \ell(\w)}{\partial w_1} = - \frac{\sum_{i=2}^k e^{1-w_1+w_i}}{1+\sum_{p=2}^k e^{1-w_1+w_p}} \eqdef \alpha_1
\]
and for $i \ge 2$
 \[
\frac{\partial \ell(\w)}{\partial w_i} =  \frac{e^{1-w_1+w_i}}{1+\sum_{p=2}^k e^{1-w_1+w_p}} \eqdef \alpha_i ~.
\]
Note that $-\alpha_1 = \sum_{i=2}^k \alpha_1 \le 1$, and that for $i \ge 2$,  $\alpha_i \ge 0$. Let $H$ be the Hessian of $\ell$ at $\w$. It follows that for $i \ge 2$,
{\small \[
H_{i,i} = \frac{e^{1-w_1+w_i}}{1+\sum_{p=2}^k e^{1-w_1+w_i}}  - \frac{(e^{1-w_1+w_i})^2}{1+(\sum_{p=2}^k e^{1-w_1+w_i})^2} = \alpha_i - \alpha_i^2 ~.
\]}
In addition, for $j \neq i$ where both $j$ and $i$ are not $1$ we have
\[
H_{i,j} = \frac{0 - e^{1-w_1+w_i}e^{1-w_1+w_j}}{(\sum_{p=2}^k e^{1-w_1+w_i})^2} =
- \alpha_i \alpha_j ~.
\]
For $i=1$ we have
\[
H_{1,1} = -\alpha_1 - \alpha_1^2
\]
and for $i > 1$
\[
H_{i,1} = -\alpha_i - \alpha_1 \alpha_i
\]
We can therefore rewrite $H$ as
{\small \bea
H &=& - \alpha \alpha^\trans + \textrm{diag}([-\alpha_1, \alpha_2 ,\ldots, \alpha_k])  - \e_1 [0,\alpha_1,\ldots,\alpha_k]\\
&& - [0,\alpha_1,\ldots,\alpha_k]^\trans (\e_1)^\trans ~.
\eea}
It thus follows that:
\begin{align*}
\v^\trans H \v &= -(\inner{\alpha,\v})^2 - \alpha_1 v_1^2+ \sum_{i>1} \alpha_i v_i^2 - 2 v_1 \sum_{i>1} \alpha_i v_i \\
&\le 0 + \sum_{i>1} \alpha_i (v_i^2 - v_1^2 - 2 v_1 v_i)\\
&= \sum_{i>1} \alpha_i ((v_i-v_1)^2 -2 v_1^2 )\\
&\le  2\max_i \v_i^2 = 2 \|\v\|_\infty^2 ~,
\end{align*}
where the last step is because for any $v_i \in [-c,c]$, the function $f(v_1) = (v_i-v_1)^2 - 2v_1^2$ receives its maximum when $v_1 = -v_i$ and then its value is $2v_i^2$.
This concludes our proof.
\end{proof}

The above lemma implies that $L$ is smooth in the following sense:
\begin{lemma} \label{lem:smooth}
For any $W,U$ s.t. $U = \u \,\e_r^\trans$ (that is, only the $r$'th column of $U$ is not zero) we have that
\[
L(W-U) \le L(W) - \inner{\nabla L(W),U} +  \|\u\|_\infty^2 ~.
\]
\end{lemma}
\begin{proof}
Recall that $L(W)$ is the average over $(\x,y)$ of a function of the form $\ell(W\x)$, where $\ell$ is as defined in \lemref{lem:smooth1}. Therefore,
\begin{align*}
\ell((W+U)\x) &\le \ell(W\x) + \inner{\nabla \ell(W\x),U\x} +  \|U\x\|_\infty^2 \\
&= \ell(W\x) + \inner{\nabla \ell(W\x),U\x} + |x_r|^2 \|\u\|_\infty^2 \\
&\le \ell(W\x) + \inner{\nabla \ell(W\x),U\x} +  \|\u\|_\infty^2 ~,
\end{align*}
where the last inequality is because we assume that $\|\x\|_\infty \le 1$ for all $\x$.
The above implies that
\begin{equation} \label{eqn:smoothproof1}
L(W-U) \le L(W) - \inner{\nabla L(W),U} +  \|\u\|_\infty^2 ~.
\end{equation}
\end{proof}

Equipped with the smoothness property of $L$, we now turn to show that  if the greedy algorithm has not yet identified all the features of $W^\star$ then a single greedy iteration
yields a substantial progress.  We use the notation $\supp(W)$ to denote the indices of columns of $W$ which are not all-zeros.
\begin{lemma} \label{lma:key}
Let $F,\bF$ be two subsets of $[d]$ such that $\bF-F \neq \emptyset$ and let
$$
W = \argmin_{V: \supp(V)=F} L(V) ~~~~,~~~~
W^\star = \argmin_{V: \supp(V)=\bF} L(V)  ~~.
$$
Then,  if $L(W) > L(W^\star)$ we have
$$
L(W) - \min_{\u} L(W + \u\e_j^\trans) ~\geq~
\frac{\left(L(W)-L(W^\star) \right)^2}{4\,
\left(\sum_{i \in \bF-F}\|W^\star_{\cdot,i}\|_\infty\right)^2} ,
$$
where $j=  \argmax_{i} \|\nabla_i L(W)\|_1$.
\end{lemma}
\begin{proof}
To simplify notation, denote $F^c = \bF-F$.
Using \lemref{lem:smooth}  we know that for any $\u$:
\[
L(W-\u \e_j^\trans) \le L(W) - \inner{\nabla L(W),\u \e_j^\trans} + \|\u\|_\infty^2 ~,
\]
In particular, the above holds for the vector of $\u = \tfrac{1}{2} \| \nabla_{j} L(W)\|_1 \, \sgn(\nabla_j L(W))$ and by rearranging we obtain that
\begin{align*}
L(W) - L(W-\u \e_j^\trans) &\ge \inner{\nabla L(W),\u \e_j^\trans} -  \|\u\|_\infty^2 \\
&= \tfrac{1}{4} \|\nabla_j L(W)\|_1^2 ~.
\end{align*}
It is therefore suffices to show that
\[
\frac{1}{4} \|\nabla_j L(W)\|_1^2 ~\ge~
\frac{\left(L(W)-L(W^\star) \right)^2}{4 \left(\sum_{i \in \bF-F}\|W^\star_{\cdot,i}\|_\infty\right)^2} ~.
\]
Denote $s = \sum_{j \in F^c} \|W^\star_{\cdot,j}\|_\infty$, then an equivalent inequality\footnote{This is indeed equivalent because the lemma assumes that $L(W) > L(W^\star)$ } is
\[
s\, \|\nabla_j L(W)\|_1 ~\ge~
L(W)-L(W^\star) ~.
\]
From the convexity of $L$, the right-hand side of the above is upper bounded by $\inner{\nabla L(W),W-W^\star}$. Hence, it is left to show that
\[
s\, \|\nabla_j L(W)\|_1 ~\ge~ \inner{\nabla L(W),W - W^\star} ~.
\]
Since we assume that $W$ is optimal over $F$ we get that
$\nabla_i L(W) = \mathbf{0}$ for all $i \in F$, hence $\inner{\nabla L(W),W} = 0$. Additionally,  $W^\star_{\cdot,i} = \mathbf{0}$ for $i \not\in \bF$. Therefore,
\begin{align*}
 \inner{\nabla L(W),W-W^\star } &= - \sum_{i \in F^c} \inner{\nabla_i L(W),W^\star_{\cdot,i} } \\
 &\le  \sum_{i \in F^c} \|\nabla_i L(W) \|_1 \, \|W^\star_{\cdot,i} \|_\infty \\
  &\le s \, \max_i  \|\nabla_i L(W) \|_1 \\
  &=  s \, \|\nabla_j L(W)\|_1  ~,
\end{align*}
and this concludes our proof.
\end{proof}

Using the above lemma, the proof of our main theorem easily follows.
\begin{proof}[of \thmref{thm:sparse}]
Denote $\epsilon_t = L(W^{(t)}) - L(W^\star)$, where $W^{(t)}$ is the value of $W$ at iteration $t$.
The definition of the update implies that $L(W^{(t+1)}) \leq
\min_{i,\u} \, L(W^{(t)} + \u \e_i^\trans)$.
The conditions of \lemref{lma:key} hold and therefore
we obtain that (with $F = F^{(t)}$)
{\small \begin{equation} \label{eqn:fully-onestep-strong}
\begin{split}
\epsilon_t - \epsilon_{t+1} &= L(W^{(t)}) - L(W^{(t+1)}) \geq~
\frac{\epsilon_t^2}{4\,
\left(\sum_{i \in \bF-F} \|W^\star_{\cdot,i}\|_\infty\right)^2} \\
&\geq~ \frac{\epsilon_t^2}{4\,\|W^\star\|_{\infty,1}^2} ~.
\end{split}
\end{equation}}
Using Lemma B.2 from \cite{ShalevSrZh10}, the above implies that
for $t \ge 4\,\|W^\star\|_{\infty,1}^2/\epsilon$ we have that $\epsilon_t \le \epsilon$, which concludes our proof.
\end{proof}

\end{document}